\documentclass{article}

 \usepackage[main,final]{neurips_2026}

\usepackage[utf8]{inputenc} 
\usepackage[T1]{fontenc}    
\usepackage{hyperref}       
\usepackage{url}            
\usepackage{booktabs}       
\usepackage{amsfonts}       
\usepackage{nicefrac}       
\usepackage{microtype}      
\usepackage{xcolor}         

\usepackage{amsmath,amssymb,amsthm}
\newtheorem{proposition}{Proposition}
\newtheorem{remark}{Remark}
\newtheorem{corollary}{Corollary}

\usepackage{algorithm}
\usepackage{algpseudocode}
\usepackage{xspace}
\usepackage{amsmath}
\usepackage{multirow}
\usepackage{caption}
\usepackage{pifont}
\usepackage{graphicx}
\usepackage{tabularx}
\usepackage{wrapfig,lipsum}
\usepackage{subfig}
\usepackage[inline]{enumitem}
\usepackage{wrapfig}

\usepackage[nameinlink]{cleveref}
\crefname{paragraph}{Sec.}{Secs.}
\Crefname{paragraph}{Sec.}{Secs.}
\crefname{section}{Sec.}{Secs.}
\crefname{table}{Tab.}{Tabs.}
\crefname{equation}{Eq.}{Eqs.}
\crefname{figure}{Fig.}{Figs.}

\newcommand{\methodAbbr}{OmniMem}

\newcommand{\onedot}{\ifx\@let@token.\else.\null\fi\xspace}

\newcommand{\eg}[0]{\emph{e.g}\onedot}

\newcommand{\para}[1]{\noindent \textbf{#1.}}

\newcommand{\lightmidrule}{\addlinespace[1pt]}

\newcommand{\appcref}[1]{Appendix~\cref{#1}}

\usepackage{listings}
\usepackage{xcolor}
\title{\methodAbbr{}: Scalable and Adaptive Memory Retrieval for Long Video Generation}

\author{%
    Lin Zhao\textsuperscript{1,2}\thanks{Equal contribution.}
    \quad
    Yushu Wu\textsuperscript{1}\footnotemark[1]
    \quad
    Yifan Gong\textsuperscript{2}
    \quad
    Yanzhi Wang\textsuperscript{1}
    \quad
    Pu Zhao\textsuperscript{1}\thanks{Corresponding author.}
    \\
    \textsuperscript{1}Northeastern University\quad
    \textsuperscript{2}Adobe Research
    \\
    Project Page: \url{https://wuyushuwys.github.io/OmniMem}
}

\begin{document}

\maketitle

\begin{figure*}[h]
\begingroup
\setlength{\abovecaptionskip}{3pt}
\setlength{\belowcaptionskip}{0pt}
  \centering
  \includegraphics[width=0.97\linewidth]{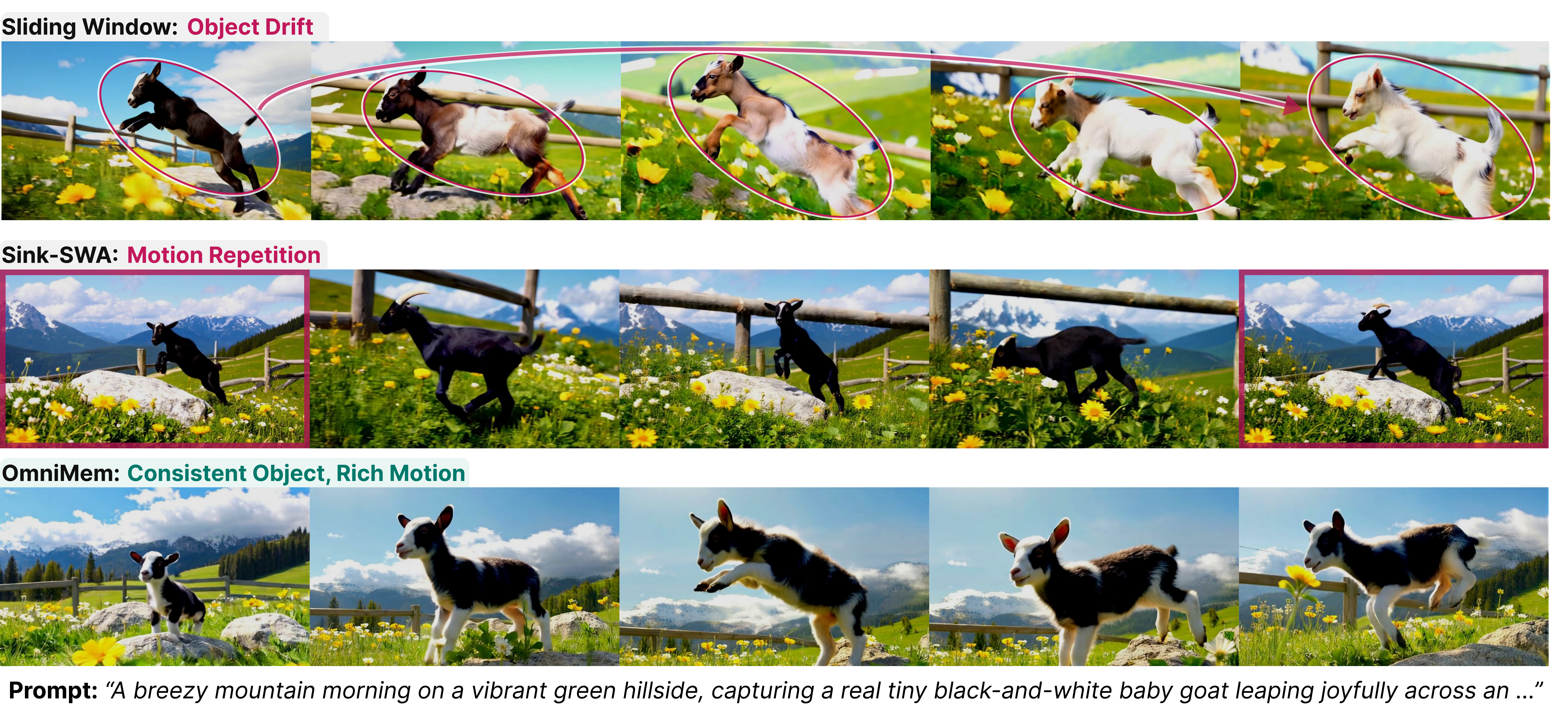}
\caption{
\textbf{\methodAbbr{} preserves object identity while maintaining rich motion in long video generation.}
SWA shows object drift, and Sink-SWA produces repetitive motion.
}
\label{fig:teaser}
\endgroup
\end{figure*}

\begin{abstract}
Autoregressive (AR) video generation extends videos by producing latent chunks sequentially, but scaling to long videos requires repeated access to a growing historical KV cache.
Existing methods reduce this cost by truncating the KV cache or compressing it into implicit memory, but both lose explicit access to query-relevant historical details.
We propose \methodAbbr{}, an explicit full-range memory retrieval framework that performs sparse KV retrieval over the historical cache.
To make this practical for chunk-based AR video generation, \methodAbbr{} addresses two issues: (i) local bias in sparse KV selection and (ii) Union Explosion in memory access.
Adaptive Window Exclusion removes local-window blocks from the selection candidates when sufficient long-range history is available, preserving the sparse budget for informative long-range retrieval.
Query-Shared KV Selection reduces cross-query diversity, while Per-Head Scattered KV Access avoids expanding head-specific selections into a large selected KV buffer.
This allows each attention head to retrieve non-contiguous KV blocks according to its own selection pattern.
Experiments on long-video generation show that \methodAbbr{} improves Dynamic Degree by 52.3\% and preserves strong consistency over strong baselines, while maintaining comparable memory usage.

\end{abstract}
\section{Introduction}

Autoregressive (AR) video generation with Diffusion Transformer (DiT) 
backbones~\cite{yang2024cogvideox,zheng2024opensora,lin2024opensoraplan,shen2026fastcar,shen2025draftattention,HunyuanVideo,polyak2024moviegen,zhao2026s2dit} 
has catalyzed significant progress in long video generation.
By generating video sequentially, AR models can go beyond the fixed 
temporal horizon of standard bidirectional 
models~\cite{wang2024loong,deng2024nova,jin2024pyramidal,yin2025slow,che2024gamegen,deng2024causalfusion,wu2025taming}.
This property is important for applications that require long, coherent visual sequences, such as interactive gaming, robotics, and world models \cite{lingbot-va2026,bruce2024genie,yin2025slow}.

Nevertheless, scaling AR video generation to long sequences introduces a fundamental computational challenge. 
As video length increases, the KV cache grows linearly with the number of generated chunks, increasing both memory consumption and the attention computation required at each step.
A simple solution is to discard distant KV states once the cache becomes too large.
However, this weakens long-range dependencies and causes the generated subject to gradually lose its identity, resulting in noticeable drift in color, shape, and trajectory across frames as illustrated in \cref{fig:teaser}.
Existing methods mainly address this drift issue in two ways.
The first retains a fixed set of anchor tokens or memory banks (\eg, sink tokens) to stabilize the generation \cite{yang2025longlive,liu2025rolling,ji2025memflow}. 
While these help preserve identity, we observe that they often lead to repetitive motion patterns in long rollouts, indicating a trade-off between consistency and motion dynamics (\cref{fig:teaser}).
The other introduces auxiliary modules (such as SSMs \cite{yu2025videossm} and Conv3D encoders \cite{zhang2025pretraining}) to compress the history into a compact state, but this may lose fine-grained details needed for long-horizon consistency.

These limitations call for a memory mechanism that can perform \textbf{full-range} sparse retrieval over the history, explicitly selecting and accessing query-relevant context on demand~\cite{yuan2025native}. 
Nevertheless, it is non-trivial to apply the sparse memory retrieval for chunk-based AR video generation models, complicated by two specific challenges: (i) \textbf{local bias}  and (ii) \textbf{union explosion}.   

\textbf{Local bias} is introduced by the attention scores to select a limited number of full-resolution KV blocks from the pooled historical cache. Since attention naturally favors nearby visual context, the attention scores for selection are dominated by nearby KV blocks, and therefore the selected blocks tend to concentrate heavily around the current chunk, which is unable to retrieve long-range historical memory. 
\textbf{Union explosion} caused by cross-query and cross-head selection divergence can undermine the memory benefit of sparse retrieval. 
Although each query retrieves only a small number of KV blocks, chunk-based AR video generation processes thousands of query tokens (\eg, 4--5K) simultaneously at each AR step, unlike LLM decoding, which typically generates one query token per step.
Different queries within the same chunk may select different historical regions due to the diverse visual content in the chunk.
Different attention heads may further prefer different historical regions because of their distinct attention patterns.
In this case, sparse retrieval reduces the attention computation for each query, but the union of distinct selected KV blocks across all queries and heads can still be large.
Although the method remains sparse in computation, it can still require a large KV memory footprint for the current AR step.
We refer to this memory-side blow-up as Union Explosion.

To address these issues, we propose \methodAbbr{}, an explicit memory retrieval framework for chunk-based AR video generation.
\methodAbbr{} addresses local bias with Adaptive Window Exclusion, which removes the KV blocks already covered by the sliding-window branch from the selection candidates, and thus preserves the sparse selection budget for informative long-range retrieval.
To address Union Explosion, \methodAbbr{} introduces Query-Shared KV Selection and Per-Head Scattered KV Access.
Query-Shared KV Selection reduces cross-query selection diversity by letting nearby query tokens share the same sparse KV selection.
Per-Head Scattered KV Access avoids materializing a large merged cache across attention heads as the video grows longer.
Instead, each attention head directly retrieves its own selected non-contiguous KV blocks from the historical cache.
Together, these components make explicit long-range retrieval practical for chunk-based AR video generation.
Extensive experiments demonstrate that \methodAbbr{} sets a new state of the art for long video generation, simultaneously improving motion dynamics and long-range consistency over prior designs, while incurring only 1.7\% additional memory over LongLive~\cite{yang2025longlive}.

Our main contributions are summarized as follows:
\begin{itemize}[leftmargin=1.5em, itemsep=1pt, topsep=1pt, parsep=1pt, partopsep=1pt]
    \item We analyze the query regime in chunk-based autoregressive video generation and identify two issues that limit explicit memory retrieval: local bias in block selection and Union Explosion caused by cross-query and cross-head divergence.
    \item We propose \methodAbbr,  combining Adaptive Window Exclusion, Query-Shared KV Selection, and Per-Head Scattered KV Access to enable explicit and scalable long-range memory retrieval.
    \item Our framework improves the Dynamic Degree score by 52.3\% over the strongest baseline on VBench-Long, while preserving visual consistency with only 1.7\% additional memory overhead.
\end{itemize}

\section{Related Work}

\para{Long Video Generation}
Recent video generation models increasingly adopt causal, AR, or streaming formulations to improve generation throughput by generating videos sequentially and reusing intermediate states across frames. Prior efforts~\cite{ai2025magi1autoregressivevideogeneration,chen2025skyreelsv2infinitelengthfilmgenerative} employ chunk-based AR generation and diffusion forcing for long-video synthesis, while distillation-based methods~\cite{yin2025slow,huang2025self} convert bidirectional models into few-step AR generators with distribution matching. To reduce error accumulation over long rollouts, LongLive~\cite{yang2025longlive} and Self-Forcing++~\cite{cui2025self} introduce rollout-aware training, Rolling Forcing~\cite{liu2025rolling} denoises within a rolling window using attention sinks as global anchors, and MMM~\cite{cai2026mmm} combines long-context flow matching with sliding-window distribution matching in a non-AR setting. These methods improve long-horizon consistency through advanced training and denoising designs, while \methodAbbr~explores a different scheme enabling explicit memory retrieval over the historical KV cache.

\para{KV Cache Compression and Memory Retrieval in Long Video Generation}
KV cache management is a key bottleneck in long-horizon video generation. In the LLM and streaming-video domains, strategies such as token eviction~\cite{zhang2023h2o,li2024snapkv}, quantization~\cite{hooper2024kvquant}, and streaming inference~\cite{xiao2023streamingllm,jiang2026dynamic} have been extensively studied and are now being adapted to autoregressive video generation, where the spatiotemporal cache grows even more aggressively. Recent approaches tackle this through low-bit KV-cache quantization~\cite{xi2026quant}, history compression via learned embeddings~\cite{zhang2025pretraining} or SSM-based global memory~\cite{yu2025videossm}, and cache sparsification, retaining only informative tokens across chunks~\cite{zhang2025blockvid} or activating relevant memory on demand~\cite{ji2025memflow}.  Further refinements exploit head-wise context redundancy~\cite{guo2026efficient} and salience-based policies that distill bidirectional knowledge for token importance estimation~\cite{yang2026pafukv}. 

\section{OmniMem: Explicit Memory Retrieval and Access}

We investigate  chunk-based Autoregressive (AR) video generation with a causal  DiT  backbone. Unlike token-wise AR generation in LLMs, DiT-based chunk generation processes thousands of  tokens simultaneously in each  chunk,   leading to high memory cost especially in long video generation. The preliminary for  AR  video generation and KV cache are detailed in \appcref{app:sec:preliminary}.

\begin{figure*}[t]
  \centering
  \includegraphics[width=0.94\linewidth]{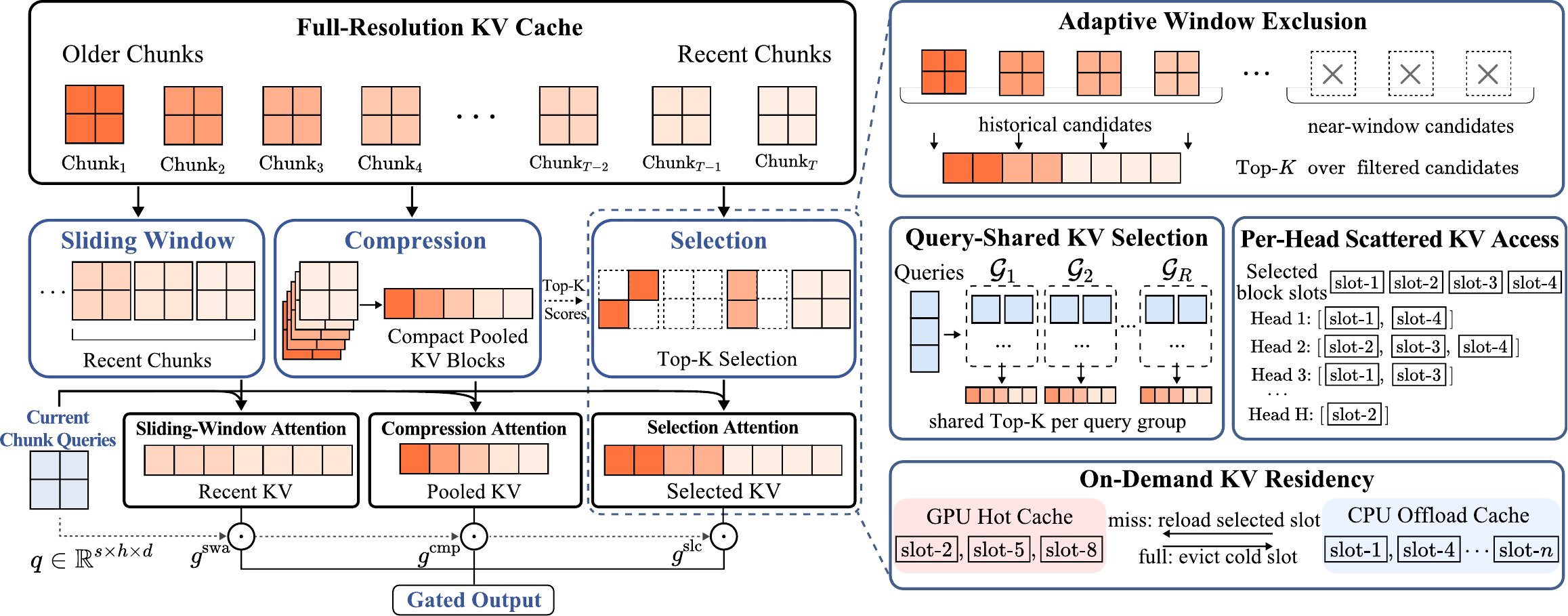}
\caption{
\textbf{Overview of \methodAbbr.}
Current-chunk queries attend to recent KV blocks, pooled historical KV blocks, and retrieved full-resolution KV blocks through sliding-window, compression, and selection attention, respectively.
The right panels summarize the key retrieval and access designs: filtering near-window candidates before Top-$K$ selection, sharing Top-$K$ selection within query groups, and accessing per-head block slots with on-demand GPU residency.
}
\label{fig:framework}
\vspace{-7pt}
\end{figure*}

\subsection{Framework}

\para{Problem Formulation}
Long-video generation with causal DiTs is bottlenecked by the growing historical KV cache. 
As shown in the visualization (\appcref{app:sec:attn_map}) of the attention maps from Self-Forcing~\cite{huang2025self}, many heads rely on scattered long-range historical positions, suggesting that the memory cannot be discarded.
This motivates us to treat long-range memory access as an explicit retrieval problem over historical context.
The goal is to preserve informative long-range dependencies while avoiding dense attention over the full history.
This formulation raises two questions: which historical blocks should be retrieved for the current chunk, and how should these selected blocks be efficiently organized  across thousands of queries with multiple attention heads.

\para{Retrieval over Historical KV Cache} 
Building on recent advances in sparse attention for LLMs, particularly Native Sparse Attention (NSA)~\cite{yuan2025native}, we formulate long-range memory access as explicit block-level retrieval over the historical KV cache:
\begin{equation}
\small
    o = \sum_{c \in \mathcal{C}} g^c \cdot \mathrm{Attn}(Q, K^c, V^c), \quad \mathcal{C} = \{\mathrm{CMP}, \mathrm{SLC}, \mathrm{SWA}\},
    \label{eq:nsa}
\end{equation}
where compression attention (CMP), selection attention (SLC), and sliding-window attention (SWA) attend to the pooled historical KV cache, the retrieved full-resolution historical KV blocks, and the recent KV cache, respectively.
Their outputs are fused by coefficients $g^c \in [0,1]$.
Selection attention uses compression-attention scores to identify the top-$K$ relevant historical KV blocks for each query.
Unlike NSA~\cite{yuan2025native}, our compression pools over 2D spatial neighborhoods to align with the visual semantics, implemented via a one-time token reorder (details and proof in \appcref{sec:reorder}).

\para{Challenges}
In practice, we find that it is non-trivial to apply sparse memory retrieval to chunk-based AR video generation models, owing to two specific challenges, including local bias and union explosion, which we analyze in the following sections. 
The two challenges correspond to two fundamental questions, respectively: which historical blocks should be retrieved, and how should these selected blocks be efficiently organized.
To answer these fundamental questions, \methodAbbr~addresses  local bias with Adaptive Window Exclusion~(\cref{sec:window_exclusive}), and union explosion through Query-Shared KV Selection and Per-Head Scattered KV Access~(\cref{sec:per-head}).

\subsection{Addressing Local Bias with Adaptive Window Exclusion}
\label{sec:window_exclusive}

We first examine which blocks are actually selected by the top-$K$ mechanism in \cref{eq:nsa}.
As shown in \cref{fig:local-bias}, the selected blocks are heavily concentrated near the current chunk, leaving most of the long-range history unselected.
This local bias makes selection attention redundant with the sliding window and weakens the purpose of explicit retrieval of historical memory.
We analyze its cause below and present an effective solution for selecting long-range blocks and jumping out of local bias.

\begingroup
\setlength{\textfloatsep}{2pt}
\begin{figure}[t]
    \captionsetup{skip=3pt}
    \centering
    \subfloat[Local bias in Top-$K$ selection]{
        \includegraphics[width=0.39\linewidth]{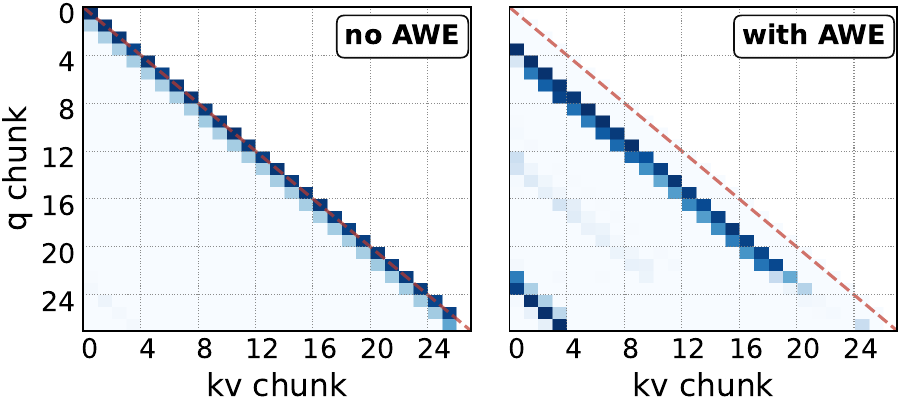}
        \label{fig:local-bias}
    }
    \hfill
    \subfloat[Cross-Query Divergence]{
        \includegraphics[width=0.271\linewidth]{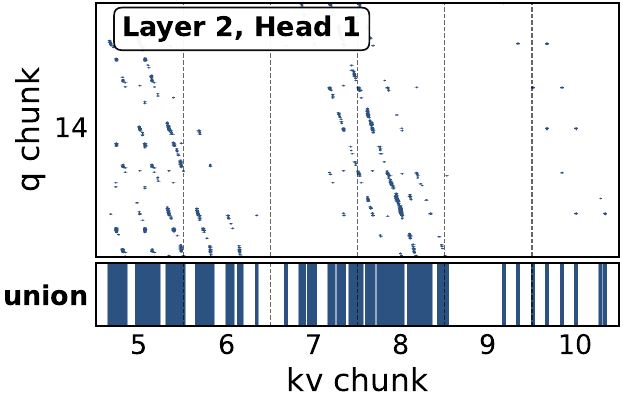}
        \label{fig:cross_query}
    }
    \hfill
    \subfloat[Cross-Head Divergence]{
        \includegraphics[width=0.271\linewidth]{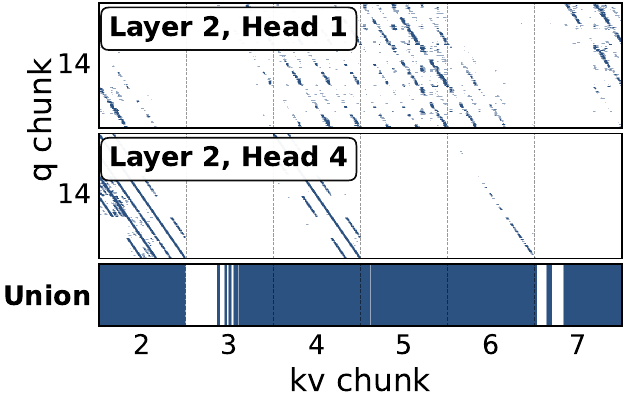}
        \label{fig:cross_head}
    }
    \caption{
    \textbf{Local bias and Union Explosion in selection attention.}
     (a) Top-K selection focuses near the current chunk without AWE, and shifts to long-range blocks with AWE. 
    (b) Different query chunks in one head select different blocks. 
    (c) Different heads also select different regions. 
    (b) and (c) together cause Union Explosion.
    Note that each chunk contains a number of tokens (\eg, 4-5$K$).
    }
    \label{fig:nsa_issues}
    \vspace{-1.em}
\end{figure}
\endgroup

\para{Local Bias in Top-$K$ Selection}
Without an explicit candidate mask, the compressed attention scores driving selection attention are dominated by nearby KV blocks. 
We attribute  local bias to three compounding factors. (i) \textit{Strong spatiotemporal correlation among nearby video tokens}. While RoPE imposes no explicit locality constraint, its relative-position parameterization can interact with this inherent visual continuity, systematically inflating attention scores between proximate tokens.
(ii) \textit{Amplification under AR video DiTs}. The bias is further exacerbated in AR DiTs, where the current chunk is denoised from a noisy latent while historical chunks supply clean, fully generated visual context. The asymmetry in signal quality between the noisy query and its immediate (clean) neighbors reinforces the tendency of attention to anchor locally.
(iii) \textit{Granularity mismatch between KV blocks and video chunks}. The KV block size is typically far smaller than a single video chunk. 
Let $S$ denote the number of tokens per chunk, $B$ the tokens per KV block, and $N_w$ the number of chunks covered by the sliding window. The local window then spans approximately $\lceil N_w S / B \rceil$ candidate blocks. Absent any mechanism to exclude these blocks, a top-$K$ selector must satisfy $K > \lceil N_w S / B \rceil$ to guarantee, in the worst case, the retrieval of even a single block outside the local window. For video chunks comprising thousands of tokens, this lower bound far exceeds practical sparse budgets, causing selection attention to collapse onto the sliding-window region rather than surface genuinely long-range memory\footnote{
For example, with $S=4680$, $N_w=3$, and $B=60$, the local window spans $\lceil 3 \times 4680 / 60 \rceil = 234$ KV blocks, far above the top-$K$ budget in practice.
}.

\para{Adaptive Window Exclusion for Selection Attention}
To mitigate  local bias in top-$K$ selection, we narrow the candidate set for selection attention to historical KV blocks outside the sliding window.
Since recent chunks are already accessed by SWA, selecting them again wastes the sparse budget of selection attention.
By removing the blocks covered by SWA from the selection candidates, selection attention is encouraged to retrieve complementary long-range memory outside the local window.

Let $\mathcal{H}^{(n)}$ denote all historical KV blocks at the $n^{th}$ AR step, and  $\mathcal{W}^{(n)} \subset \mathcal{H}^{(n)}$ denote the blocks covered by SWA.
Since selection attention retrieves top-$K$ blocks, Adaptive Window Exclusion removes $\mathcal{W}^{(n)}$ only when at least $K$ blocks remain outside the window, and top-$K$ selection is performed after Adaptive Window Exclusion:
\begin{equation}
\small
    \Omega_{\mathrm{SLC}}^{(n)}
    =
    \begin{cases}
    \mathcal{H}^{(n)} \setminus \mathcal{W}^{(n)}, 
    & \text{if } |\mathcal{H}^{(n)} \setminus \mathcal{W}^{(n)}| \ge K, \\
    \mathcal{H}^{(n)}, 
    & \text{otherwise}.
    \end{cases}
    \ \ \ \ \ \ \ \
        \mathcal{I}_{h,q}^{(n)}
    =
    \operatorname{TopK}_{b \in \Omega_{\mathrm{SLC}}^{(n)}}
    s_{h,q,b}.
    \label{eq:per_q_topk}
\end{equation}
For each query token $q$ and attention head $h$, selection attention uses compression-attention score $s_{h,q,b}$ for  historical block $b$ to select block indices.
The selected indices are used to fetch the corresponding full-resolution KV blocks.

Applying exclusion unconditionally can destabilize selection attention when the historical cache is short. If few or no valid blocks exist outside the recent window, the selection output may shift abruptly,  
which can perturb the fused representation of all attention branches and impede training stability.
To mitigate this, we make exclusion \textit{conditional on the available history}. Before a sufficient number of outside-window blocks have accumulated, selection attention operates over the full historical KV cache. Once this threshold is reached,  blocks already covered by SWA are excluded. 

\subsection{Mitigating Union Explosion with Query-Shared KV Selection and Per-Head  KV Access}
\label{sec:per-head}

While Adaptive Window Exclusion governs \emph{which} explicit memory blocks are selected, the efficiency of long-video retrieval equally depends on \emph{where} these blocks reside. We now turn to the management and access of selected blocks.
Under a memory budget,  we  maintain only a bounded set of explicit memory blocks on the GPU and offload less recently used historical blocks to CPU memory, as illustrated in \cref{fig:framework}. Under this design, the selected-block union becomes the critical quantity governing efficiency: prior to computing selection attention, all blocks chosen for the current chunk must be resident on the GPU. A larger union not only increases the volume of KV entries 
on GPU, but also inflates the temporary buffer required to materialize the selected KV, in a naïve implementation.

In LLM decoding, each step generates a single query token, so top-$K$ selection produces a small and temporally stable set of blocks. Chunk-based AR video generation departs from this regime in a fundamental way: thousands of query tokens are processed in parallel, and distinct queries or attention heads may attend to disparate historical regions. Consequently, the union of selected blocks can grow rapidly within a single chunk—a phenomenon we term \textbf{Union Explosion}. We analyze this growth along two complementary axes: \begin{enumerate*}[label=(\roman*)]
    \item cross-query divergence across queries, and
    \item cross-head divergence across attention heads
\end{enumerate*}.

\para{Cross-Query Divergence}
LLM decoding generates one query per step, so the top-$K$ selection naturally yields a small set of $K$ historical KV blocks.
AR video generation jointly produces a sequence of latent tokens within each chunk.
Different queries in the sequence carry different preferences and may select different historical KV blocks, so the chunk-level set of selected blocks is the union of
selections across all queries~(\cref{fig:cross_query}), which can grow up to $S \cdot K$ in the worst case
\footnote{
As a simple illustration, consider $S=4$ query tokens, each selecting one block from a history of 4 blocks.
If the four queries select distinct blocks, the chunk-level union covers the full history despite each per-query budget is one.
}.

\para{Cross-Head Divergence}
Popular video DiTs~\cite{LTX-video,wan2025wanopenadvancedlargescale,HunyuanVideo} use multi-head attention~(MHA), where each query head has its own KV head.
This one-to-one structure makes both the CMP scores and the retrieved KV blocks head-specific.
Empirically, as shown in ~\cref{fig:cross_head}, different heads exhibit distinct selection patterns in chunk-based AR video generation.
Some heads retrieve long-range context, while others focus on recent chunks.
This motivates head-specific retrieval rather than relying only on fixed global memory, such as sink tokens \cite{yang2025longlive,liu2025rolling}.

To effectively address the above challenges, we propose Query-Shared KV Selection for cross-query divergence, and Per-Head Scattered KV Access for cross-head divergence, as detailed below.

\para{Query-Shared KV Selection}
Cross-query divergence manifests at the query-token level. 
The spatially and temporally adjacent query tokens exhibit strong local correlation, tending to  retrieve overlapping sets of historical KV blocks. 
To effectively mitigate this, we propose Query-Shared KV Selection, which  exploits this local correlation explicitly by assigning a single shared 
KV-block selection to each group of adjacent query tokens, with 
independent selection across various groups. 

Formally, the $S$ query tokens are partitioned into $R=\lceil S/G_q\rceil$ adjacent query groups, where $G_q$ denotes the number of query tokens per group. All query tokens within the same group share a unified top-$K$ block selection, effectively constraining intra-group selection diversity without sacrificing the representational flexibility of inter-group retrieval. Thus,   the per-query selection in \cref{eq:per_q_topk} becomes
\begin{equation}
\small
    \bar{s}_{h,r,b}
    =
    \operatorname{AvgPool}_{q\in\mathcal{G}_r}
    s_{h,q,b},
    \quad
    \mathcal{I}_{h,r}^{(n)}
    =
    \operatorname{TopK}_{b \in \Omega_{\mathrm{SLC}}^{(n)}}
    \bar{s}_{h,r,b},
\end{equation}
where $\mathcal{G}_r$ denotes the set of token indices belonging to group  $r$. This reduces the worst-case selected block union along the query dimension from $S\cdot K$ to $R\cdot K$, while preserving diversity across query groups.
It also makes the selected KV access more regular, allowing the selection-attention kernel to use larger dense tiles and improve GPU Tensor Core utilization (details in \appcref{sec:supp:impl}).

\para{Per-Head Scattered KV Access}
As shown in~\cref{fig:cross_head}, different heads often prefer different historical regions, so forcing all heads to share the same selected blocks under a fixed sparse budget may reduce the expressiveness.
Therefore, we perform selection independently for each head.

The challenge is how to represent these head-specific selections efficiently.
Under the offload-and-reload setting, selected full-resolution KV blocks must be resident on GPU before selection attention is computed.
A straightforward implementation gathers these blocks into a temporary selected-KV buffer for the attention kernel.
Specifically, following the query sharing above, let $\mathcal{I}_{h,r}^{(n)}$ denote the top-$K$ block indices selected by query group $r$ for head $h$.
The blocks accessed by head $h$ are
\begin{equation}
\small
    \mathcal{U}_{h}^{(n)}
    =
    \bigcup_{r=1}^{R}
    \mathcal{I}_{h,r}^{(n)},
    \qquad
    R=\left\lceil \frac{S}{G_q} \right\rceil .
\end{equation}
The cross-head union determines which physical KV blocks must be resident on GPU or reloaded from CPU memory.
The bottleneck here is not the attention computation itself, 
but rather the selected-KV representation required by the attention kernel.
Materializing this union uniformly across all heads, however, introduces a large number of unused KV slots, since each head in practice attends to only a subset of the union.
If implemented by materializing a head-aligned selected-KV buffer, the footprint scales as 
$   M_{\mathrm{merge}}
    \propto
    N_h
    \left|
    \bigcup_{h=1}^{N_h}
    \mathcal{U}_{h}^{(n)}
    \right|.
$

\methodAbbr~instead uses Per-Head Scattered KV Access, as shown in~\cref{fig:framework}.
The full-resolution KV cache is organized as per-head, chunk-level tensors, and a pointer table maps the indices of selected blocks for each head to addresses within this native storage. 
The selection attention kernel consumes the per-head selected block indices, reading the corresponding non-contiguous KV blocks directly.
This design strictly bounds memory access to the exact blocks required by each head.
In the offload-and-reload setting, avoiding the materialization of a padded, cross-head dense tensor saves massive, redundant PCIe data transfers and prevents GPU memory exhaustion.
As a result, the selected-KV representation scales with the sum of per-head selections, rather than with the union across heads (see \appcref{sec:supp:per_head_access} for details):
\begin{equation}
\small
    M_{\mathrm{scatter}}
    \propto
    \sum_{h=1}^{N_h}
    \left|
    \mathcal{U}_{h}^{(n)}
    \right|.
    \label{eq:mem_scatter_kv}
\end{equation}
This avoids materializing a cross-head selected-KV buffer while preserving head-specific retrieval.

We highlight that \methodAbbr{} is specifically designed for chunk-based AR video generation, which significantly differs from sparse retrieval  \cite{yuan2025native} in LLMs, with more details above and in \appcref{app:sec:gqa}.

\section{Experiments}

\para{Implementation Details}
\methodAbbr~is built on Wan2.1-T2V-1.3B~\cite{wan2025wanopenadvancedlargescale}.
Following prior autoregressive video generation pipelines~\cite{yin2025slow,huang2025self}, we add our memory module and fine-tune the model with ODE initialization on VidProM prompts~\cite{wang2024vidprom}.
We then continue training with Self-Forcing~\cite{huang2025self} and long-video tuning, following LongLive~\cite{yang2025longlive}.
For the memory module, we use $15 \times 2$ pooling blocks, top-$12$ retrieval per query group, and a query-group size of 15.
Adaptive Window Exclusion removes the most recent 3 chunks from selection candidates once enough history is available.
The selection-attention kernel is implemented in Triton.
More details are provided in \appcref{sec:supp:impl}.

\subsection{Main Results}
\begin{figure*}[t]
  \centering
  \includegraphics[width=1.0\linewidth]{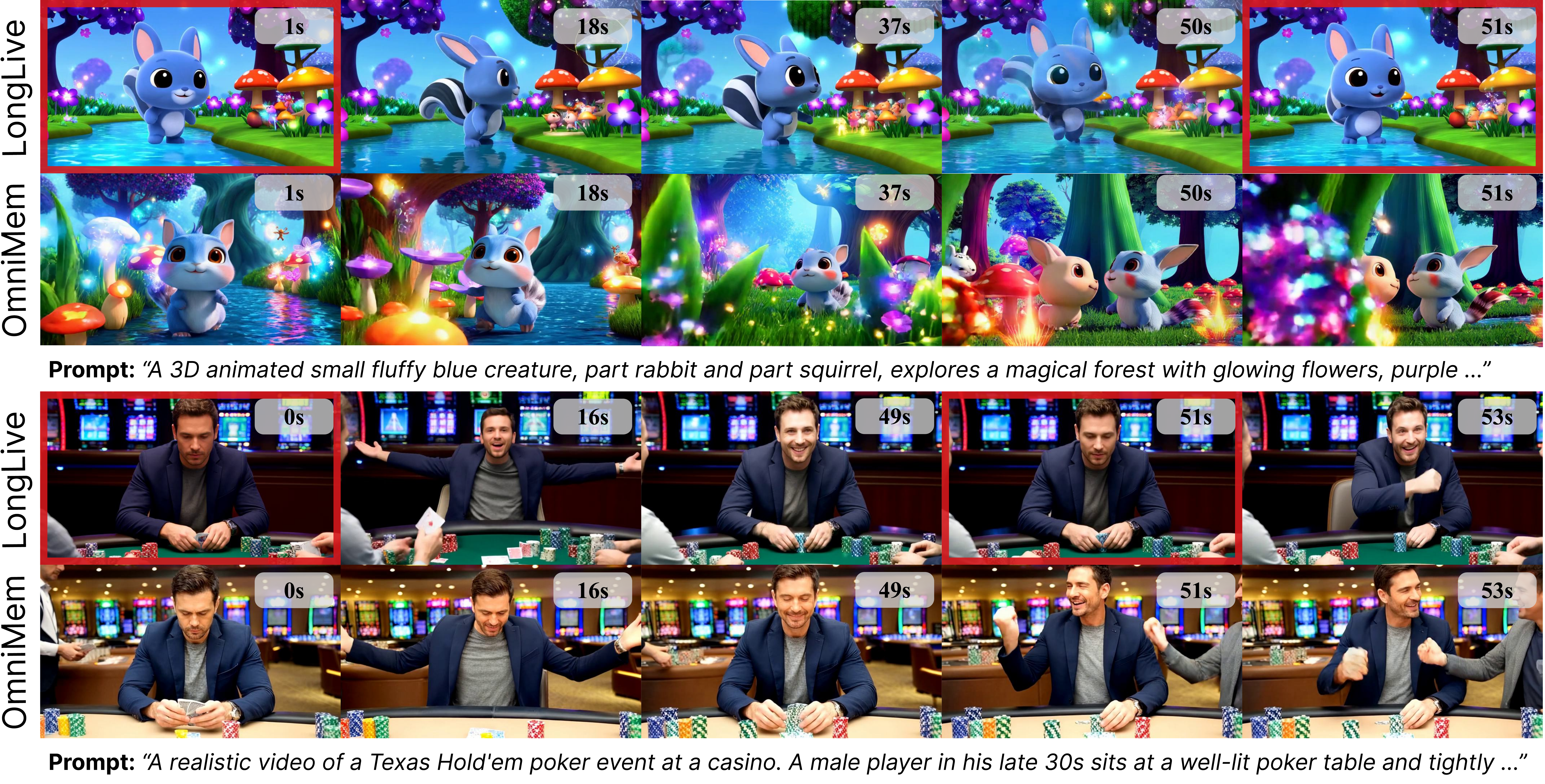}
\caption{
\textbf{Qualitative comparison on long-video generation.} Red boxes highlight repetitive frames where LongLive \cite{yang2025longlive} collapses back to early content. Full videos and additional results are provided in the supplementary material.
}
\label{fig:qualitative_results}
\vspace{-4pt}
\end{figure*}

\begin{table*}[t]
\centering
\caption{\textbf{Quantitative comparison on 60-second long video generation under VBench-Long~\cite{huang2025vbench++}.} 
\textbf{Bold} denotes the best. VRAM measured on the Nvidia H100.}
\label{tab:long-vbench}
\resizebox{0.95\linewidth}{!}{%
\begin{tabular}{l ccccccc}
\toprule
\multirow{2}{*}{Model} & \multirow{2}{*}{VRAM~(GB)} & Total  & Quality & Semantic & Subject & Imaging & Dynamic \\
& & Score  & Score & Score & Consistency & Quality & Degree \\
\midrule
SkyReels-V2~\cite{chen2025skyreelsv2infinitelengthfilmgenerative}   &   28.7    & 79.70   & 82.89   & 66.97   & 94.99   & 60.01   & 53.19 \\
RollingForcing~\cite{liu2025rolling}   &   25.3    & 82.46   & 84.34   & 74.95   & 94.97   & \textbf{72.22}   & 47.64 \\
LongLive~\cite{yang2025longlive}    &   22.9    & 82.28   & 83.79   & 76.23   & 95.00   & 69.94   & 53.47 \\
MemFlow~\cite{ji2025memflow}      &    26.2     & 82.56   & 84.14   & 76.24   & 94.49   & 70.67   & 54.03 \\
\midrule
\textbf{\methodAbbr}             &    23.3      & \textbf{83.08} & \textbf{84.74} & \textbf{76.46} & \textbf{95.09}    & 71.49    & \textbf{82.29}    \\
\bottomrule
\end{tabular}%
}
\vspace{-2pt}
\end{table*}

\para{Long Video Generation}
We evaluate 60-second single-prompt video generation under the official VBench-Long~\cite{huang2025vbench++} setting to measure long-horizon quality, consistency, and motion dynamics.
We compare against the representative methods that target for solving long video problems as shown in~\cref{tab:long-vbench}. 
\methodAbbr{} achieves the best overall score and improves most long-video metrics, with a particularly large gain in Dynamic Degree.
Prior methods obtain Dynamic Degree scores below $55$, indicating that they tend to produce limited motion under long rollouts.
In contrast, \methodAbbr{} achieves $82.29$ of Dynamic Degree, indicating stronger motion dynamics under the same long-video setting.
In terms of memory, \methodAbbr{} uses only 1.7\% more VRAM than LongLive (sink+SWA) and remains more memory-efficient than other baselines, while preserving strong consistency.
The qualitative comparison in \cref{fig:qualitative_results} shows a similar trend.
Compared with LongLive~\cite{yang2025longlive}, \methodAbbr{} maintains more stable visual details and richer motion over time, while LongLive shows limited temporal variation in the highlighted regions.
Additional qualitative results are provided in \appcref{sec:supp:more_qualitative}.

\begin{table*}[t]
\centering
\caption{\textbf{Quantitative comparison on multi-prompt 60-second long video generation.} Imaging Quality (IQ), Subject Consistency (SC), and Background Consistency (BC) are computed under VBench-Long~\cite{huang2025vbench++}. CLIP scores are computed per $10$-second segment to assess prompt adherence.}
\label{tab:multi-prompt}
\setlength{\tabcolsep}{4pt}
\renewcommand{\arraystretch}{1.}
\resizebox{\linewidth}{!}{%
\begin{tabular}{l ccc cccccc}
\toprule
\multirow{2}{*}{Model} & Imaging & Subject & Background & \multicolumn{6}{c}{CLIP Score$\uparrow$} \\
\cline{5-10}
 & Quality & Consistency & Consistency & 0--10\,s & 10--20\,s & 20--30\,s & 30--40\,s & 40--50\,s & 50--60\,s \\
\midrule
LongLive~\cite{yang2025longlive}   & 71.94 & 97.09 & 96.08 & \textbf{27.51} & 25.85 & 25.67 & 25.26 & 25.15 & 24.78 \\
MemFlow~\cite{ji2025memflow}       & 71.44 & \textbf{97.78} & 96.51 & 27.35 & 25.93 & 25.79 & 25.69 & 25.11 & 24.67 \\
\midrule
\textbf{\methodAbbr}               & \textbf{73.85} & 97.72 & \textbf{96.55} & 27.43 & \textbf{26.07} & \textbf{26.03} & \textbf{25.89} & \textbf{25.72} & \textbf{24.99} \\
\bottomrule
\end{tabular}%
}
\end{table*}

\para{Long Video Generation with Prompt Switching}
We further evaluate \methodAbbr~ under a multi-prompt long video generation setting, where the input is a sequence of $6$ successive prompts, each spanning $10$ seconds, that together specify a $60$-second video.
Following LongLive~\cite{yang2025longlive}, we construct an evaluation set of $100$ multi-prompt scripts.
We then evaluate the methods that support prompt switching on this set, following VBench-Long~\cite{huang2025vbench++}.
To assess prompt adherence over time, we additionally compute CLIP scores between each generated 10-second segment and its corresponding prompt.
In \cref{tab:multi-prompt}, \methodAbbr{} achieves the highest IQ and BC, with competitive SC.
More importantly, \methodAbbr~maintains stable CLIP scores as the video length grows, with only a $2.44$ drop from $0$--$10$s to $50$--$60$s.
The results indicate that \methodAbbr~preserves quality, consistency, and video-text alignment throughout long generation.
At 120 seconds, \methodAbbr{} also outperforms LongLive on imaging quality, subject consistency, and background consistency under prompt switching (\appcref{app:sec:long120}).

Results on short video generation are provided in \appcref{app:sec:short_video}.

\para{Human Evaluation}
We compare \methodAbbr{}, LongLive, and MemFlow on 50 randomly sampled prompt-switching cases, each producing a 60-second video.
Fifteen evaluators each assess all 50 anonymous, randomly shuffled video triplets, selecting the best video separately for text alignment, overall quality, and long-range consistency (750 judgments per criterion).
\methodAbbr{} receives the highest preference rate on all three criteria (\cref{tab:human-eval}); the protocol and criterion definitions are provided in \appcref{app:sec:human_eval}.

\para{Comparison with Adapted NSA}
To isolate the retrieval design, we compare with an NSA baseline using the same backbone and training recipe on 60-second prompt-switching generation.
We replace NSA's 1D pooling with spatial 2D pooling and enable per-head KV management and CPU offloading during training to avoid out-of-memory failures on H200.
The baseline retains NSA's block selection, without Adaptive Window Exclusion or Query-Shared KV Selection.
\methodAbbr{} improves all three quality metrics in \cref{tab:nsa-comparison}.
For inference memory profiling, NSA runs without CPU offloading on H200.
Full protocol details are in \appcref{app:sec:nsa_protocol}.

\begin{table}[t]
\begin{minipage}[t]{0.48\linewidth}
\centering
\caption{\textbf{Human preference rates (\%).} Higher is better.}
\label{tab:human-eval}
\small
\setlength{\tabcolsep}{4pt}
\begin{tabular}{lccc}
\toprule
Method & \shortstack{Text\\alignment} & \shortstack{Overall\\quality} & \shortstack{Long-range\\consistency} \\
\midrule
LongLive & 14.9 & 12.4 & 25.9 \\
MemFlow & 24.3 & 21.9 & 18.0 \\
\methodAbbr{} & \textbf{60.8} & \textbf{65.7} & \textbf{56.1} \\
\bottomrule
\end{tabular}
\end{minipage}\hfill
\begin{minipage}[t]{0.49\linewidth}
\centering
\caption{\textbf{Adapted NSA on 60-second videos.} Peak inference VRAM is measured on an H200.}
\label{tab:nsa-comparison}
\small
\setlength{\tabcolsep}{3pt}
\begin{tabular}{lrrrr}
\toprule
Method & \shortstack{VRAM\\(GB)} & IQ & SC & BC \\
\midrule
Adapted NSA & 138.4 & 71.32 & 95.23 & 95.74 \\
\methodAbbr{} & \textbf{23.3} & \textbf{73.85} & \textbf{97.72} & \textbf{96.55} \\
\bottomrule
\end{tabular}
\end{minipage}
\end{table}

\subsection{Ablation Study}

We ablate individual components on 20-second videos with $K=4$, reporting Imaging Quality (IQ), Subject Consistency (SC), Background Consistency (BC), and CLIP score.

\para{Impact of Memory Branches}
We ablate the contribution of each attention branch by progressively disabling CMP and SLC, with results reported in \cref{tab:abl-fused}.
Using SWA alone yields the weakest results, while adding either CMP or SLC brings consistent improvements.
Combining all three achieves the best performance, showing that coarse global memory (CMP) and fine-grained retrieval (SLC) are both essential complements to the local window (SWA).

\begingroup
\setlength{\textfloatsep}{0pt}
\setlength{\intextsep}{0pt}

\begin{table*}[t]
\centering

\begin{minipage}[t]{0.48\textwidth}
\centering
\captionof{table}{\textbf{Ablation on the three attention branches of the fused design.} 
}
\label{tab:abl-fused}
\setlength{\tabcolsep}{8pt}
\renewcommand{\arraystretch}{1.}
\resizebox{\linewidth}{!}{%
\begin{tabular}{lcccc}
\toprule
\textbf{Variant} & \textbf{IQ}  & \textbf{SC}  & \textbf{BC}  & \textbf{CLIP} \\
\midrule
SWA                    & 66.74 & 96.76 & 95.85 & 24.05 \\
CMP + SWA              & 69.31 & 97.89 & 96.80 & 24.89 \\ 
SLC + SWA              & 71.13 & 97.85 & 97.02 & 25.37 \\
\lightmidrule
\textbf{CMP + SLC + SWA}                   & \textbf{71.38} & \textbf{98.34} & \textbf{97.30} & \textbf{26.53} \\
\bottomrule
\end{tabular}
}
\end{minipage}
\hfill
\begin{minipage}[t]{0.47\textwidth}
\centering
\captionof{table}{
\textbf{Comparison of different window exclusion chunk size.} $^\dagger$ means the default setting.
}
\label{tab:abl-prog}
\setlength{\tabcolsep}{8pt}
\renewcommand{\arraystretch}{1.}
\resizebox{\linewidth}{!}{%
\begin{tabular}{lcccc}
\toprule
\textbf{Excluded Chunks} & \textbf{IQ} & \textbf{SC} & \textbf{BC} & \textbf{CLIP} \\
\midrule
0 chunks              & 66.13 & 96.78 & 96.52 & 24.51 \\
1 chunk               & 68.95 & 97.77 & 96.58 & 24.67 \\
2 chunks              & \textbf{71.41} & 97.90 & 96.71 & 26.38 \\
\lightmidrule
\textbf{3 chunks~$^\dagger$}    & 71.38 & \textbf{98.34} & \textbf{97.30} & \textbf{26.53} \\
\bottomrule
\end{tabular}
}
\end{minipage}
\vspace{-10pt}
\end{table*}

\endgroup

\para{Impact of Window Exclusion Size}
As discussed in \cref{sec:window_exclusive}, Adaptive Window Exclusion reduces the local bias in AR video generation.
We vary the number of excluded chunks to study how much of the sliding-window region should be removed from the selection candidates.
As shown in \cref{tab:abl-prog}, excluding more recent chunks substantially improves performance over no exclusion.
Removing the full SWA window ($3$-chunk) yields the best overall result and is used as our default.
This supports our design of narrowing the selection candidates to historical KV blocks outside the local window.

\begingroup
\setlength{\textfloatsep}{0pt}
\setlength{\intextsep}{0pt}

\begin{table*}[t]
\centering

\begin{minipage}[t]{0.32\textwidth}
\centering
\captionof{table}{
\textbf{Ablation on query-group size $G_q$.} 
$^\dagger$: default.
}
\label{tab:diff_q_share}
\setlength{\tabcolsep}{5pt}
\resizebox{\linewidth}{!}{%
\begin{tabular}{lcccc}
\toprule
\textbf{$G_q$} & \textbf{IQ} & \textbf{SC}  & \textbf{BC} & \textbf{CLIP} \\
\midrule
1              & 71.43 & 98.52 & 97.24 & 26.47 \\
15 $^\dagger$  & 71.38 & 98.34 & 97.30 & 26.53 \\
30             & 70.14 & 98.34 & 97.01 & 25.98 \\
60             & 69.52 & 98.26 & 96.93 & 24.80 \\
\bottomrule
\end{tabular}
}
\end{minipage}
\hfill
\begin{minipage}[t]{0.31\textwidth}
\centering
\captionof{table}{
\textbf{Ablation on head-group size $G_h$.}
}
\label{tab:abl-group}
\setlength{\tabcolsep}{4pt}
\renewcommand{\arraystretch}{1.15}
\resizebox{\linewidth}{!}{%
\begin{tabular}{lcccc}
\toprule
$G_h$ & \textbf{IQ}  & \textbf{SC}  & \textbf{BC}  & \textbf{CLIP}  \\
\midrule
12                      & 69.34 & 97.13 & 96.56 & 23.64 \\
3                      & 70.80 & 97.85 & 96.82 & 26.29 \\
\lightmidrule
\textbf{1}                     & \textbf{71.38} & \textbf{98.34} & \textbf{97.30} & \textbf{26.53} \\
\bottomrule
\end{tabular}
}
\end{minipage}
\hfill
\begin{minipage}[t]{0.32\textwidth}
\centering
\captionof{table}{
\textbf{Ablation on sink tokens and selection attention.}
}
\label{tab:analysis_sink}
\setlength{\tabcolsep}{2.5pt}
\renewcommand{\arraystretch}{1.15}
\small
\resizebox{\linewidth}{!}{%
\begin{tabular}{lcccc}
\toprule
\textbf{Variant} & \textbf{IQ}  & \textbf{SC}  & \textbf{BC} & \textbf{CLIP} \\
\midrule
SWA                & 66.74 & 96.76 & 95.85 & 24.05 \\
SWA+Sink           & 69.14 & 96.97 & 96.90 & 25.32 \\
SWA+SLC+Sink    & 70.50 & 97.76 & \textbf{97.25} & 25.29 \\
\lightmidrule
\textbf{SWA+ SLC}           & \textbf{71.13} & \textbf{97.85} & 97.02 & \textbf{25.37} \\
\bottomrule
\end{tabular}
}
\end{minipage}
\hfill
\vspace{-1em}
\end{table*}

\endgroup

\para{Impact of Query-Shared KV Selection}
We vary the group size $G_q$ to demonstrate the effect of query-shared selection.
\cref{tab:diff_q_share} indicates that increasing $G_q$ from 1 to 15 leaves all metrics nearly unchanged, confirming that adjacent queries within a chunk share similar context needs and can be safely grouped.
Therefore, query-shared KV selection mitigates the cross-query Union Explosion without sacrificing generation quality, and additionally improves training efficiency by reducing redundant selection computation.
Larger groups reduce IQ, BC, and CLIP; SC is unchanged at 30 and decreases at 60, suggesting that one shared selection becomes less effective for larger groups.

\para{Impact of Per-Head KV Access}
An alternative way to eliminate cross-head divergence is to group the selection across different heads, so that grouped heads retrieve the same historical blocks.
We ablate this design in \cref{tab:abl-group}, where the head-group size $G_h$ denotes the number of heads sharing a common KV selection.
Sharing the selection across all 12 heads significantly degrades all metrics, and even a moderate size of $G_h=3$ still leaves a clear gap to per-head selection.
This indicates that different attention heads attend to genuinely different historical regions, and forcing them to share a single selection hurts generation quality.

Additional ablations on pooling size, Top-$K$, and LRU cache size are deferred to \appcref{app:sec:ablate}.

\subsection{Scalability of Long-Range Memory Access}
\label{sec:efficiency}
\begin{figure*}[t]
  \centering
  \subfloat{
    \includegraphics[width=0.45\linewidth]{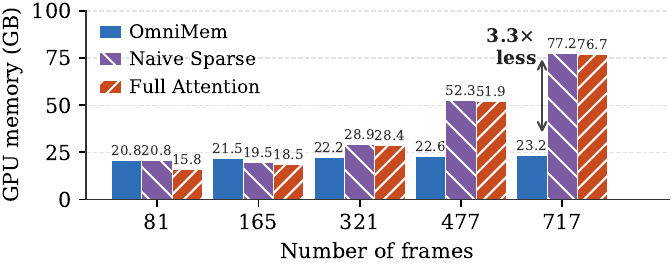}
  }
  \hspace{0.03\textwidth}
  \subfloat{
    \includegraphics[width=0.45\linewidth]{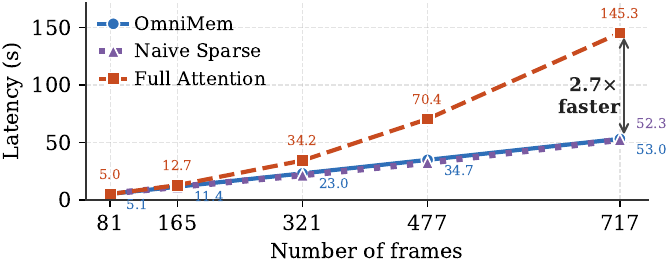}
  }
\caption{
\textbf{Memory access scalability.}
Naive Sparse reduces latency but loses the memory benefit due to Union Explosion.
\methodAbbr{} maintains memory usage nearly constant while remaining efficient.
}
\label{fig:efficiency_com}
\vspace{-3pt}
\end{figure*}

\cref{fig:efficiency_com} compares \methodAbbr{} with full-history attention and a Naive Sparse baseline, which applies per-query, per-head Top-$K$ retrieval without Query-Shared KV Selection or Per-Head Scattered KV Access.
Although Naive Sparse keeps the attention computation sparse, its selected-block union spans a large fraction of historical chunks.
To preserve low latency, these chunks must stay resident on GPU, causing memory to grow close to full-history attention.
In contrast, \methodAbbr{} keeps the active KV set bounded through query sharing and per-head scattered access.
At 717 frames, \methodAbbr{} uses $3.3\times$ less memory and runs $2.7\times$ faster than full-history attention, matching Naive Sparse latency.
\section{Discussion and Further Analysis}
\label{sec:discussion}
\begin{figure*}[t]
  \centering
  \includegraphics[width=\linewidth]{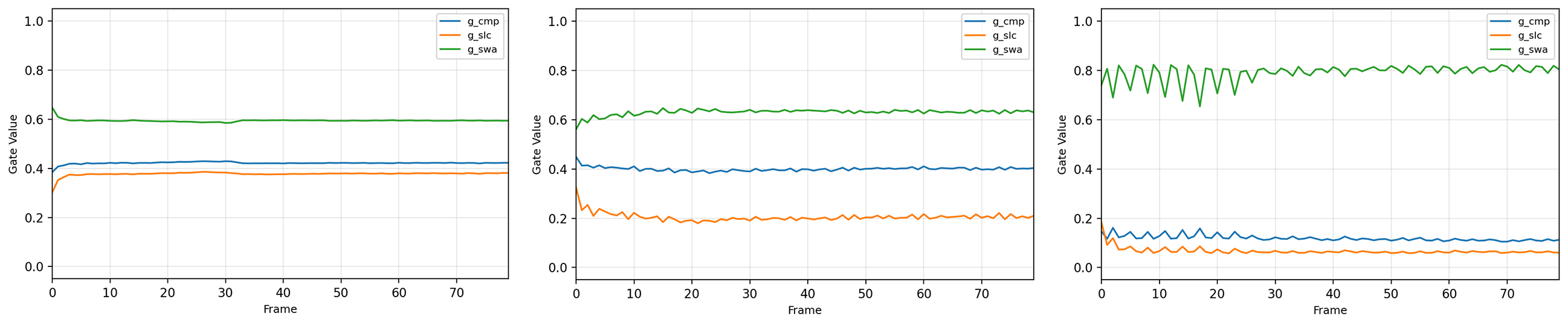}
\caption{
\textbf{Fusion gate weights across attention heads.} We visualize the gate values of the compression ($g_{\text{cmp}}$), selection ($g_{\text{slc}}$), and sliding-window ($g_{\text{swa}}$) branches across frames for three representative attention heads. 
}
\label{fig:gate_ac_head}
\end{figure*}

\para{Dynamic Retrieval over Fixed Sink Tokens}
Sink-based long-video methods keep fixed early tokens or chunks as global anchors, motivated by the attention-sink pattern.
However, \cref{fig:attn_self_forcing} shows that many heads use historical context beyond the sink region.
In contrast, \methodAbbr{} dynamically retrieves full-resolution history, so sink-like regions can still be selected when useful without restricting retrieval to fixed anchors.
Adding fixed sink anchors lowers IQ, SC, and CLIP while slightly improving BC (\cref{tab:analysis_sink}), showing no consistent benefit across metrics.

\para{Selection Attention Uses Historical Tokens More Effectively}
We further compare SWA+SLC with SWA+Sink under the same recent-window context.
SWA+SLC selects full-resolution historical blocks per query and head, rather than always reserving fixed anchors.
As shown in \cref{tab:analysis_sink}, it achieves better performance with fewer historical tokens, indicating a more effective use of historical memory.

\para{Semantic Relevance of Retrieved History}
We analyze retrieval on 40 prompt-switching videos at layers 2, 15, and 29.
Frequently retrieved historical frames have higher CLIP alignment with the current prompt than lightly retrieved frames.
Within the same historical frame, selected blocks also have higher DINOv2 similarity to the current query group than unselected blocks (e.g., $0.776$ versus $0.367$ at layer 15).
These observations support the semantic relevance of retrieved history; the protocol and mean $\pm$ standard deviation results are reported in \appcref{app:sec:retrieval_relevance}.

\para{Different Heads Prefer Different Memory Branches}
Compression summarizes global history, selection retrieves full-resolution details, and sliding-window attention retains recent context.
The fusion gates of three representative heads in \cref{fig:gate_ac_head} show distinct preferences: some balance the three branches, while others favor sliding-window attention.
These patterns suggest head-level specialization in the use of historical information.

\para{Length Extrapolation without Long-Video Finetuning}
Explicit retrieval applies the same selection rule as the historical candidate pool grows.
\Cref{fig:wo_finetune} shows a 20-second video generated by a model trained only on 5-second clips, without long-video finetuning.
The example provides qualitative evidence of visual consistency beyond the training horizon.
Our main 60-second evaluations use long-video finetuning.

\section{Conclusion}

We presented \methodAbbr, an explicit memory retrieval framework for chunk-based AR video generation.
Instead of relying solely on SWA and sink tokens, \methodAbbr~retrieves selected full-resolution historical KV blocks to enable long-range memory access.
To make this practical, \methodAbbr~addresses local-biased block selection and Union Explosion with Adaptive Window Exclusion, Query-Shared KV Selection, and Per-Head Scattered KV Access.
Experiments on 60-second generation show that \methodAbbr~improves long-video quality, temporal consistency, and motion dynamics over SOTA baselines while maintaining practical inference latency and memory footprint.

\clearpage
\begin{ack}
This work is partially supported by the National Science Foundation under Award 2414652. We would like to express our sincere gratitude to the reviewers for their invaluable feedback and constructive comments to improve the paper.
\end{ack}

{
\small
\bibliographystyle{unsrt}
\bibliography{reference}

\begin{thebibliography}{10}

\bibitem{yang2024cogvideox}
Zhuoyi Yang, Jiayan Teng, Wendi Zheng, Ming Ding, Shiyu Huang, Jiazheng Xu, Yuanming Yang, Wenyi Hong, Xiaohan Zhang, Guanyu Feng, et~al.
\newblock Cogvideox: Text-to-video diffusion models with an expert transformer.
\newblock {\em arXiv preprint arXiv:2408.06072}, 2024.

\bibitem{zheng2024opensora}
Zangwei Zheng, Xiangyu Peng, Tianji Yang, Chenhui Shen, Shenggui Li, Hongxin Liu, Yukun Zhou, Tianyi Li, and Yang You.
\newblock Open-sora: Democratizing efficient video production for all.
\newblock {\em arXiv preprint arXiv:2412.20404}, 2024.

\bibitem{lin2024opensoraplan}
Bin Lin, Yunyang Ge, Xinhua Cheng, Zongjian Li, Bin Zhu, Shaodong Wang, Xianyi He, Yang Ye, Shenghai Yuan, Liuhan Chen, et~al.
\newblock Open-sora plan: Open-source large video generation model.
\newblock {\em arXiv preprint arXiv:2412.00131}, 2024.

\bibitem{shen2026fastcar}
Xuan Shen, Weize Ma, Yufa Zhou, Enhao Tang, Yanyue Xie, Zhengang Li, Yifan Gong, Quanyi Wang, Henghui Ding, Yiwei Wang, Pu~Zhao, Jun Lin, and Jiuxiang Gu.
\newblock Fastcar: Cache attentive replay for fast auto-regressive video generation on the edge.
\newblock In {\em The Fourteenth International Conference on Learning Representations}, 2026.

\bibitem{shen2025draftattention}
Xuan Shen, Chenxia Han, Yufa Zhou, Yanyue Xie, Yifan Gong, Quanyi Wang, Yiwei Wang, Yanzhi Wang, Pu~Zhao, and Jiuxiang Gu.
\newblock Draftattention: Fast video diffusion via low-resolution attention guidance.
\newblock {\em arXiv preprint arXiv:2505.14708}, 2025.

\bibitem{HunyuanVideo}
Weijie Kong, Qi~Tian, Zijian Zhang, Rox Min, Zuozhuo Dai, Jin Zhou, Jiangfeng Xiong, Xin Li, Bo~Wu, Jianwei Zhang, et~al.
\newblock Hunyuanvideo: A systematic framework for large video generative models.
\newblock {\em arXiv preprint arXiv:2412.03603}, 2024.

\bibitem{polyak2024moviegen}
Adam Polyak, Amit Zohar, Andrew Brown, Andros Tjandra, Animesh Sinha, Ann Lee, Apoorv Vyas, et~al.
\newblock Movie gen: A cast of media foundation models.
\newblock {\em arXiv preprint arXiv:2410.13720}, 2024.

\bibitem{zhao2026s2dit}
Lin Zhao, Yushu Wu, Aleksei Lebedev, Dishani Lahiri, Meng Dong, Arpit Sahni, Michael Vasilkovsky, Hao Chen, Ju~Hu, Aliaksandr Siarohin, et~al.
\newblock S2dit: Sandwich diffusion transformer for mobile streaming video generation.
\newblock {\em arXiv preprint arXiv:2601.12719}, 2026.

\bibitem{wang2024loong}
Yuqing Wang, Tianwei Xiong, Daquan Zhou, Zhijie Lin, Yang Zhao, Bingyi Kang, Jiashi Feng, and Xihui Liu.
\newblock Loong: Generating minute-level long videos with autoregressive language models.
\newblock {\em arXiv preprint arXiv:2410.02757}, 2024.

\bibitem{deng2024nova}
Haoge Deng, Ting Pan, Haiwen Diao, Zhengxiong Luo, Yufeng Cui, Huchuan Lu, Shiguang Shan, Yonggang Qi, and Xinlong Wang.
\newblock Autoregressive video generation without vector quantization.
\newblock {\em arXiv preprint arXiv:2412.14169}, 2024.

\bibitem{jin2024pyramidal}
Yang Jin, Zhicheng Sun, Ningyuan Li, Kun Xu, Kun Xu, Hao Jiang, Nan Zhuang, Quzhe Huang, Yang Song, Yadong Mu, and Zhouchen Lin.
\newblock Pyramidal flow matching for efficient video generative modeling.
\newblock {\em arXiv preprint arXiv:2410.05954}, 2024.

\bibitem{yin2025slow}
Tianwei Yin, Qiang Zhang, Richard Zhang, William~T Freeman, Fredo Durand, Eli Shechtman, and Xun Huang.
\newblock From slow bidirectional to fast autoregressive video diffusion models.
\newblock In {\em Proceedings of the IEEE/CVF Conference on Computer Vision and Pattern Recognition}, pages 22963--22974, 2025.

\bibitem{che2024gamegen}
Haoxuan Che, Xuanhua He, Quande Liu, Cheng Jin, and Hao Chen.
\newblock Gamegen-x: Interactive open-world game video generation.
\newblock In {\em International Conference on Learning Representations}, 2025.

\bibitem{deng2024causalfusion}
Chaorui Deng, Deyao Zhu, Kunchang Li, Shi Guang, and Haoqi Fan.
\newblock Causal diffusion transformers for generative modeling.
\newblock {\em arXiv preprint arXiv:2412.12095}, 2024.

\bibitem{wu2025taming}
Yushu Wu, Yanyu Li, Anil Kag, Ivan Skorokhodov, Willi Menapace, Ke~Ma, Arpit Sahni, Ju~Hu, Aliaksandr Siarohin, Dhritiman Sagar, et~al.
\newblock Taming diffusion transformer for efficient mobile video generation in seconds.
\newblock {\em arXiv preprint arXiv:2507.13343}, 2025.

\bibitem{lingbot-va2026}
Lin Li, Qihang Zhang, Yiming Luo, Shuai Yang, Ruilin Wang, Fei Han, Mingrui Yu, Zelin Gao, Nan Xue, Xing Zhu, Yujun Shen, and Yinghao Xu.
\newblock Causal world modeling for robot control.
\newblock {\em arXiv preprint arXiv:2601.21998}, 2026.

\bibitem{bruce2024genie}
Jake Bruce, Michael Dennis, Ashley Edwards, Jack Parker-Holder, Yuge Shi, Edward Hughes, Matthew Lai, Aditi Mavalankar, Richie Steigerwald, Chris Apps, Yusuf Aytar, Sarah Bechtle, Feryal Behbahani, Stephanie Chan, Nicolas Heess, Lucy Gonzalez, Simon Osindero, Sherjil Ozair, Scott Reed, Jingwei Zhang, Konrad Zolna, Jeff Clune, Nando de~Freitas, Satinder Singh, and Tim Rocktäschel.
\newblock Genie: Generative interactive environments.
\newblock {\em Forty-first International Conference on Machine Learning}, 2024.

\bibitem{yang2025longlive}
Shuai Yang, Wei Huang, Ruihang Chu, Yicheng Xiao, Yuyang Zhao, Xianbang Wang, Muyang Li, Enze Xie, Yingcong Chen, Yao Lu, Song Han, and Yukang Chen.
\newblock Longlive: Real-time interactive long video generation.
\newblock {\em arXiv preprint arXiv:2509.22622}, 2025.

\bibitem{liu2025rolling}
Kunhao Liu, Wenbo Hu, Jiale Xu, Ying Shan, and Shijian Lu.
\newblock Rolling forcing: Autoregressive long video diffusion in real time.
\newblock {\em arXiv preprint arXiv:2509.25161}, 2025.

\bibitem{ji2025memflow}
Sihui Ji, Xi~Chen, Shuai Yang, Xin Tao, Pengfei Wan, and Hengshuang Zhao.
\newblock Memflow: Flowing adaptive memory for consistent and efficient long video narratives.
\newblock {\em arXiv preprint arXiv:2512.14699}, 2025.

\bibitem{yu2025videossm}
Yifei Yu, Xiaoshan Wu, Xinting Hu, Tao Hu, Yangtian Sun, Xiaoyang Lyu, Bo~Wang, Lin Ma, Yuewen Ma, Zhongrui Wang, et~al.
\newblock Videossm: Autoregressive long video generation with hybrid state-space memory.
\newblock {\em arXiv preprint arXiv:2512.04519}, 2025.

\bibitem{zhang2025pretraining}
Lvmin Zhang, Shengqu Cai, Muyang Li, Chong Zeng, Beijia Lu, Anyi Rao, Song Han, Gordon Wetzstein, and Maneesh Agrawala.
\newblock Pretraining frame preservation in autoregressive video memory compression.
\newblock {\em arXiv preprint arXiv:2512.23851}, 2025.

\bibitem{yuan2025native}
Jingyang Yuan, Huazuo Gao, Damai Dai, Junyu Luo, Liang Zhao, Zhengyan Zhang, Zhenda Xie, Yuxing Wei, Lean Wang, Zhiping Xiao, et~al.
\newblock Native sparse attention: Hardware-aligned and natively trainable sparse attention.
\newblock In {\em Proceedings of the 63rd Annual Meeting of the Association for Computational Linguistics (Volume 1: Long Papers)}, pages 23078--23097, 2025.

\bibitem{ai2025magi1autoregressivevideogeneration}
Sand. ai, Hansi Teng, Hongyu Jia, Lei Sun, Lingzhi Li, Maolin Li, Mingqiu Tang, Shuai Han, Tianning Zhang, W.~Q. Zhang, Weifeng Luo, Xiaoyang Kang, Yuchen Sun, Yue Cao, Yunpeng Huang, Yutong Lin, Yuxin Fang, Zewei Tao, Zheng Zhang, Zhongshu Wang, Zixun Liu, Dai Shi, Guoli Su, Hanwen Sun, Hong Pan, Jie Wang, Jiexin Sheng, Min Cui, Min Hu, Ming Yan, Shucheng Yin, Siran Zhang, Tingting Liu, Xianping Yin, Xiaoyu Yang, Xin Song, Xuan Hu, Yankai Zhang, and Yuqiao Li.
\newblock Magi-1: Autoregressive video generation at scale, 2025.

\bibitem{chen2025skyreelsv2infinitelengthfilmgenerative}
Guibin Chen, Dixuan Lin, Jiangping Yang, Chunze Lin, Junchen Zhu, Mingyuan Fan, Hao Zhang, Sheng Chen, Zheng Chen, Chengcheng Ma, Weiming Xiong, Wei Wang, Nuo Pang, Kang Kang, Zhiheng Xu, Yuzhe Jin, Yupeng Liang, Yubing Song, Peng Zhao, Boyuan Xu, Di~Qiu, Debang Li, Zhengcong Fei, Yang Li, and Yahui Zhou.
\newblock Skyreels-v2: Infinite-length film generative model, 2025.

\bibitem{huang2025self}
Xun Huang, Zhengqi Li, Guande He, Mingyuan Zhou, and Eli Shechtman.
\newblock Self forcing: Bridging the train-test gap in autoregressive video diffusion.
\newblock {\em arXiv preprint arXiv:2506.08009}, 2025.

\bibitem{cui2025self}
Justin Cui, Jie Wu, Ming Li, Tao Yang, Xiaojie Li, Rui Wang, Andrew Bai, Yuanhao Ban, and Cho-Jui Hsieh.
\newblock Self-forcing++: Towards minute-scale high-quality video generation.
\newblock {\em arXiv preprint arXiv:2510.02283}, 2025.

\bibitem{cai2026mmm}
Shengqu Cai, Weili Nie, Chao Liu, Julius Berner, Lvmin Zhang, Nanye Ma, Hansheng Chen, Maneesh Agrawala, Leonidas Guibas, Gordon Wetzstein, and Arash Vahdat.
\newblock Mode seeking meets mean seeking for fast long video generation.
\newblock In {\em arXiv}, 2026.

\bibitem{zhang2023h2o}
Zhenyu Zhang, Ying Sheng, Tianyi Zhou, Tianlong Chen, Lianmin Zheng, Ruisi Cai, Zhao Song, Yuandong Tian, Christopher R{\'e}, Clark Barrett, et~al.
\newblock H2o: Heavy-hitter oracle for efficient generative inference of large language models.
\newblock {\em Advances in Neural Information Processing Systems}, 36:34661--34710, 2023.

\bibitem{li2024snapkv}
Yuhong Li, Yingbing Huang, Bowen Yang, Bharat Venkitesh, Acyr Locatelli, Hanchen Ye, Tianle Cai, Patrick Lewis, and Deming Chen.
\newblock Snapkv: Llm knows what you are looking for before generation.
\newblock {\em arXiv preprint arXiv:2404.14469}, 2024.

\bibitem{hooper2024kvquant}
Coleman Hooper, Sehoon Kim, Hiva Mohammadzadeh, Michael~W. Mahoney, Yakun~Sophia Shao, Kurt Keutzer, and Amir Gholami.
\newblock Kvquant: Towards 10 million context length llm inference with kv cache quantization.
\newblock {\em Advances in Neural Information Processing Systems}, 37, 2024.

\bibitem{xiao2023streamingllm}
Guangxuan Xiao, Yuandong Tian, Beidi Chen, Song Han, and Mike Lewis.
\newblock Efficient streaming language models with attention sinks.
\newblock {\em arXiv preprint arXiv:2309.17453}, 2023.

\bibitem{jiang2026dynamic}
Xinru Jiang, Lin Zhao, Xi~Xiao, Yunbei Zhang, Janet Wang, Chenrui Ma, Haolin Li, Yanzhi Wang, Yifan Gong, and Octavia Camps.
\newblock Dynamic hub-and-spoke memory for streaming video understanding.
\newblock In {\em Findings of the {A}ssociation for {C}omputational {L}inguistics: {EMNLP} 2026}, pages 47455--47475, 2026.

\bibitem{xi2026quant}
Haocheng Xi, Shuo Yang, Yilong Zhao, Muyang Li, Han Cai, Xingyang Li, Yujun Lin, Zhuoyang Zhang, Jintao Zhang, Xiuyu Li, et~al.
\newblock Quant videogen: Auto-regressive long video generation via 2-bit kv-cache quantization.
\newblock {\em arXiv preprint arXiv:2602.02958}, 2026.

\bibitem{zhang2025blockvid}
Zeyu Zhang, Shuning Chang, Yuanyu He, Yizeng Han, Jiasheng Tang, Fan Wang, and Bohan Zhuang.
\newblock Blockvid: Block diffusion for high-quality and consistent minute-long video generation.
\newblock {\em arXiv preprint arXiv:2511.22973}, 2025.

\bibitem{guo2026efficient}
Hang Guo, Zhaoyang Jia, Jiahao Li, Bin Li, Yuanhao Cai, Jiangshan Wang, Yawei Li, and Yan Lu.
\newblock Efficient autoregressive video diffusion with dummy head.
\newblock {\em arXiv preprint arXiv:2601.20499}, 2026.

\bibitem{yang2026pafukv}
Xu~Yang et~al.
\newblock Past- and future-informed kv cache policy with salience estimation in autoregressive video diffusion.
\newblock {\em arXiv preprint arXiv:2601.21896}, 2026.

\bibitem{LTX-video}
Yoav HaCohen, Nisan Chiprut, Benny Brazowski, Daniel Shalem, Dudu Moshe, Eitan Richardson, Eran Levin, Guy Shiran, Nir Zabari, Ori Gordon, et~al.
\newblock Ltx-video: Realtime video latent diffusion.
\newblock {\em arXiv preprint arXiv:2501.00103}, 2024.

\bibitem{wan2025wanopenadvancedlargescale}
Team Wan, Ang Wang, Baole Ai, Bin Wen, Chaojie Mao, Chen-Wei Xie, Di~Chen, Feiwu Yu, Haiming Zhao, Jianxiao Yang, Jianyuan Zeng, Jiayu Wang, Jingfeng Zhang, Jingren Zhou, Jinkai Wang, Jixuan Chen, Kai Zhu, Kang Zhao, Keyu Yan, Lianghua Huang, Mengyang Feng, Ningyi Zhang, Pandeng Li, Pingyu Wu, Ruihang Chu, Ruili Feng, Shiwei Zhang, Siyang Sun, Tao Fang, Tianxing Wang, Tianyi Gui, Tingyu Weng, Tong Shen, Wei Lin, Wei Wang, Wei Wang, Wenmeng Zhou, Wente Wang, Wenting Shen, Wenyuan Yu, Xianzhong Shi, Xiaoming Huang, Xin Xu, Yan Kou, Yangyu Lv, Yifei Li, Yijing Liu, Yiming Wang, Yingya Zhang, Yitong Huang, Yong Li, You Wu, Yu~Liu, Yulin Pan, Yun Zheng, Yuntao Hong, Yupeng Shi, Yutong Feng, Zeyinzi Jiang, Zhen Han, Zhi-Fan Wu, and Ziyu Liu.
\newblock Wan: Open and advanced large-scale video generative models, 2025.

\bibitem{wang2024vidprom}
Wenhao Wang and Yi~Yang.
\newblock Vidprom: A million-scale real prompt-gallery dataset for text-to-video diffusion models.
\newblock {\em Advances in Neural Information Processing Systems}, 37:65618--65642, 2024.

\bibitem{huang2025vbench++}
Ziqi Huang, Fan Zhang, Xiaojie Xu, Yinan He, Jiashuo Yu, Ziyue Dong, Qianli Ma, Nattapol Chanpaisit, Chenyang Si, Yuming Jiang, Yaohui Wang, Xinyuan Chen, Ying-Cong Chen, Limin Wang, Dahua Lin, Yu~Qiao, and Ziwei Liu.
\newblock {VBench++}: Comprehensive and versatile benchmark suite for video generative models.
\newblock {\em IEEE Transactions on Pattern Analysis and Machine Intelligence}, 2025.

\bibitem{huang2024vbench}
Ziqi Huang, Yinan He, Jiashuo Yu, Fan Zhang, Chenyang Si, Yuming Jiang, Yuanhan Zhang, Tianxing Wu, Qingyang Jin, Nattapol Chanpaisit, et~al.
\newblock Vbench: Comprehensive benchmark suite for video generative models.
\newblock In {\em Proceedings of the IEEE/CVF Conference on Computer Vision and Pattern Recognition}, pages 21807--21818, 2024.

\bibitem{su2024roformer}
Jianlin Su, Murtadha Ahmed, Yu~Lu, Shengfeng Pan, Wen Bo, and Yunfeng Liu.
\newblock Roformer: Enhanced transformer with rotary position embedding.
\newblock {\em Neurocomputing}, 568:127063, 2024.

\bibitem{CLIP}
Alec Radford, Jong~Wook Kim, Chris Hallacy, Aditya Ramesh, Gabriel Goh, Sandhini Agarwal, Girish Sastry, Amanda Askell, Pamela Mishkin, Jack Clark, et~al.
\newblock Learning transferable visual models from natural language supervision.
\newblock In {\em International Conference on Machine Learning}, pages 8748--8763. PMLR, 2021.

\bibitem{dinov2}
Maxime Oquab, Timothée Darcet, Théo Moutakanni, Huy Vo, Marc Szafraniec, Vasil Khalidov, Pierre Fernandez, Daniel Haziza, Francisco Massa, Alaaeldin El-Nouby, Mahmoud Assran, Nicolas Ballas, Wojciech Galuba, Russell Howes, Po-Yao Huang, Shang-Wen Li, Ishan Misra, Michael Rabbat, Vasu Sharma, Gabriel Synnaeve, Hu~Xu, Hervé Jegou, Julien Mairal, Patrick Labatut, Armand Joulin, and Piotr Bojanowski.
\newblock Dinov2: Learning robust visual features without supervision, 2024.
\newblock arXiv:2304.07193.

\end{thebibliography}
}

\appendix
\renewcommand{\thefigure}{\thesection\arabic{figure}}
\renewcommand{\thetable}{\thesection\arabic{table}}
\renewcommand{\theequation}{\thesection\arabic{equation}}

\makeatletter
\@addtoreset{figure}{section}
\@addtoreset{table}{section}
\@addtoreset{equation}{section}
\makeatother

\clearpage

\section*{Limitation and Broader Impact}
Our Triton selection-attention kernel is tuned for current GPUs, and the latency and memory gains may vary across hardware platforms.

By making long-video generation more efficient, \methodAbbr{} may reduce the compute cost of creative tools, simulation, and world-model research.
At the same time, higher-quality long video generation can be misused for disinformation or impersonation, a risk shared with general text-to-video systems.
Practical deployments should be paired with safeguards such as content filtering, watermarking, and dataset controls.

\section{Preliminaries}
\label{app:sec:preliminary}

\para{Autoregressive Video Generation}
We consider chunk-based AR video generation with a causal  DiT  backbone.
Given a conditioning signal $c$ (\eg, a text prompt), the model generates a video as a sequence of latent chunks
$\{\mathbf{z}^{(1)}, \mathbf{z}^{(2)}, \dots, \mathbf{z}^{(N)}\}$.
At the $n^{th}$ AR step, the next chunk is generated conditioned on $c$ and all previously generated chunks,
\begin{equation}
\small
\mathbf{z}^{(<n)} = \{\mathbf{z}^{(1)}, \dots, \mathbf{z}^{(n-1)}\},
\qquad
\mathbf{z}^{(n)} \sim p_\theta\!\left(\mathbf{z}^{(n)} \mid c, \mathbf{z}^{(<n)}\right),
\end{equation}
where $p_\theta$ denotes the conditional generation model.
Each chunk contains $T_c$ latent frames with spatial resolution $H_s \times W_s$.
After flattening the spatiotemporal dimensions, one chunk yields $S = T_c H_s W_s$ latent tokens.
Unlike token-wise autoregressive generation in LLMs, DiT-based chunk generation processes all $S$ tokens in the current chunk jointly at each denoising step.
As a result, a single autoregressive step introduces thousands of query tokens simultaneously, leading to a substantially different attention and memory access pattern from single-token decoding.

\para{KV Cache}
At the $n^{th}$ AR step, attention is computed over both the current chunk and all previously generated chunks, with causality enforced across chunks.
For efficiency, the keys and values of previous chunks are cached and reused.
For attention layer $l$ and head $h$, we denote the full causal KV context as
\begin{equation}
\small
\mathbf{K}_{l,h}^{(\le n)} = [\mathbf{K}_{l,h}^{(1)}; \dots; \mathbf{K}_{l,h}^{(n)}], \quad
\mathbf{V}_{l,h}^{(\leq n)} = [\mathbf{V}_{l,h}^{(1)}; \dots; \mathbf{V}_{l,h}^{(n)}].
\end{equation}
Here, the first $n\!-\!1$ entries correspond to the historical KV cache, while the last entry comes from the current chunk.
As $n$ increases, the historical cache grows linearly with the number of generated chunks.
Dense attention over the full historical cache incurs increasing computation and memory footprint, making long-video generation difficult to scale.

\section{Implementation Details}
\label{sec:supp:impl}
\para{Base Model and Training}
We build OmniMem on top of Wan2.1-T2V-1.3B~\cite{wan2025wanopenadvancedlargescale} and follow a three-stage recipe.
We first attach the memory module and fine-tune with ODE initialization on the VidProM prompts, following CausVid~\cite{yin2025slow}.
We then run Self-Forcing~\cite{huang2025self} to reduce the train-test gap during AR rollout, and finally apply the long-video tuning recipe from LongLive~\cite{yang2025longlive}.

\para{Training Hyperparameters}
The ODE stage runs for 10K steps with a global batch size of 128 and a learning rate of $1\!\times\!10^{-5}$.
AdamW is used with $\beta = (0.9, 0.999)$.
The Self-Forcing stage runs for 600 steps and the long-video tuning stage runs for 500 steps, both with a global batch size of 64.
For these two stages, we follow the optimizer settings of Self-Forcing and LongLive, using a learning rate of $4\!\times\!10^{-7}$ for the generator and $2\!\times\!10^{-6}$ for the fake score, with AdamW $\beta = (0, 0.999)$ and $0.01$ weight decay.
Training for the main OmniMem model uses 4 nodes with $8\times$ NVIDIA H100 GPUs, with gradient accumulation. The adapted NSA training protocol is described separately in \appcref{app:sec:nsa_protocol}.

\para{Memory Module}
The pooling block size is $15 \times 2$, and the compression branch is a plain mean pool over each spatial neighborhood.
The query-group size is 15, and each group selects the top-12 historical blocks per head for selection attention.
The sliding window covers the current chunk and the two most recent historical chunks.
Adaptive Window Exclusion removes sliding-window blocks from the selection candidates once sufficient historical context is available (\cref{eq:per_q_topk}).
The compression, selection, and sliding-window outputs are combined through learned per-branch sigmoid gates.
The gate parameters are initialized with a normal distribution and trained jointly with the rest of the memory module, starting from the ODE stage.

\para{Triton Kernel and Tile Padding}
We implement both the forward and backward selection attention as Triton kernels (Triton 3.5.1), so that the entire training pipeline runs on the same sparse path as inference.
The kernel reads non-contiguous KV blocks via the pointer table described in ~\cref{sec:per-head} and computes attention using an online softmax.
For efficient use of the H100 tensor cores, the tile size along the sequence axis should be a power of two to match the matrix-multiply shapes the tensor cores expect.
Our pooling block covers $15 \times 2 = 30$ tokens, which is not a power of two.
We therefore pad each tile to 32 tokens inside the kernel and mask the two extra positions in the softmax, so that the padded entries do not affect the attention output.
Padding adds $2/30\approx6.7\%$ more token positions. The resulting tile utilization is $30/32\approx93.8\%$, where utilization denotes the fraction of padded positions occupied by valid tokens.
The head dimension is already 128 and needs no padding.
A full sweep over other pooling shapes is reported in~\cref{tab:bench_pooling_size}.

\para{Inference and Offloading}
The efficiency numbers in ~\cref{sec:efficiency} are measured on a single H100 GPU.
The GPU hot cache keeps 7 chunks resident and evicts cold ones with an LRU policy.
We ablate this size in ~\cref{sec:lru-ablation}.
Offloaded KV blocks live in pinned host memory and are reloaded through PCIe when their chunks are selected.
To keep PCIe transfers efficient, we explicitly bind each GPU process to the NUMA~(Non-Uniform Memory Access) node that owns its PCIe root, thereby avoiding cross-socket traffic during reload.

\section{Shared Chunk-level KV cache and Per-Head Selection Forward}
\label{sec:supp:per_head_access}
\begin{figure}[h]
    \centering
    \includegraphics[width=0.4\linewidth]{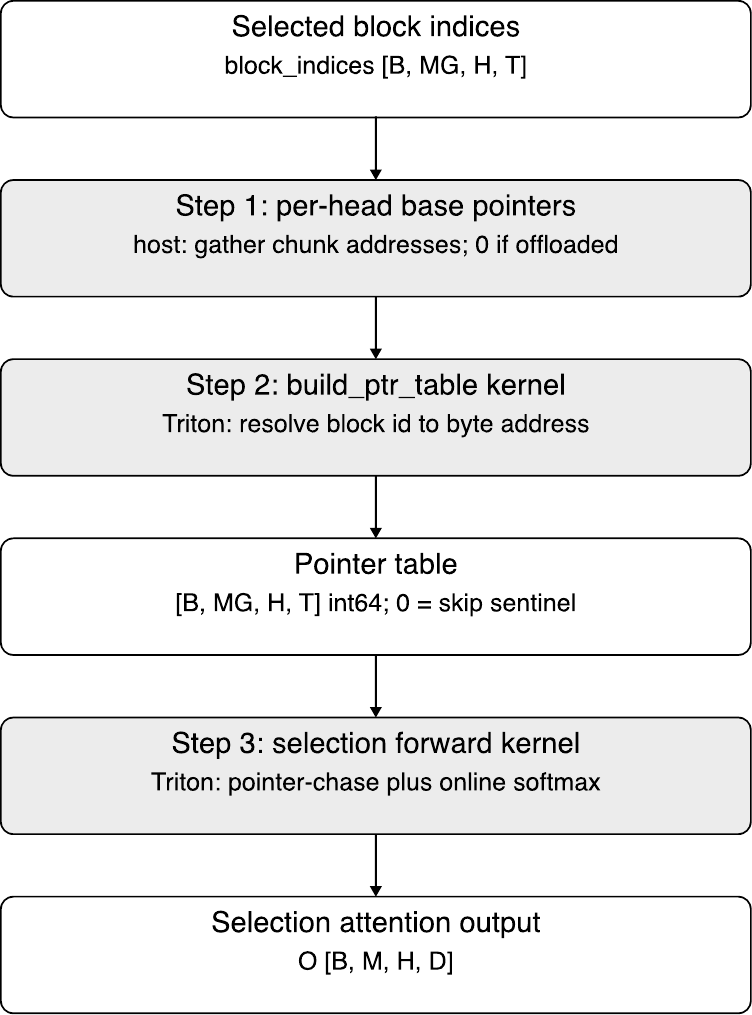}
    \caption{
    \textbf{Implementation pipeline of selection attention.}
    Step 1 runs on the host and collects the GPU addresses of the KV chunks each head needs. 
    Step 2 is a Triton kernel that turns each selected block index into its final byte address, building the pointer table.
    Step 3 is the selection forward kernel. It follows each pointer in the table to read the corresponding KV block from its original storage, and skips any entry marked with 0. Only addresses pass through the pipeline; the KV data itself stays in its per-head chunk storage.
    }
    \label{fig:supp:per_head_pipeline}
\end{figure}

\methodAbbr~maintains two levels of KV cache.
The first level is a compact compressed cache for the compression branch, and the second level is a full-resolution chunk-level cache shared by the sliding-window and selection branches.

\para{Compressed Cache}
For each generated chunk, we pool its KV states into compressed KV blocks and append them to a contiguous compressed history.
Since this cache is much shorter than the full-resolution KV cache, keeping it as a contiguous sequence makes compression attention simple and efficient.

\para{Full-resolution Chunk-level Cache}
The second level is a full-resolution chunk-level KV cache shared by the sliding-window branch and the selection branch.
For each layer and head, the cache is stored as a list of chunk tensors,
$\{\mathbf{K}_{l,h}^{(1)}, \ldots, \mathbf{K}_{l,h}^{(n)}\}$ and
$\{\mathbf{V}_{l,h}^{(1)}, \ldots, \mathbf{V}_{l,h}^{(n)}\}$.
Sliding-window attention accesses the most recent $W$ chunks from this cache.
Because the window is contiguous and fixed-size, this requires concatenating only $W$ large chunk tensors.
A block-level physical cache would instead require concatenating $W\lceil S/B\rceil$ small tensors for the same window, where $S$ is the number of tokens per chunk and $B$ is the KV block size.
Both layouts access the same tokens and require no padding, but the many-small-tensor concatenation introduces higher overhead in our list-based implementation.

\para{Selection Access via Pointer Table}
Selection attention uses the same full-resolution chunk-level cache, but accesses it through block indices.
Each selected block index is mapped to a chunk index and an intra-chunk block offset.
For each attention head, the sparse attention kernel receives a variable-length list of selected block indices and directly reads the corresponding KV blocks through a pointer table.
The pointer table itself is built by a lightweight GPU kernel that maps each (chunk index, intra-chunk offset) pair to the physical address of the corresponding KV blocks.
Only \texttt{int64} addresses are gathered, so this step carries address information rather than the raw KV data.
An entry of $0$ in the pointer table means that the chunk is on CPU, and the selection forward kernel skips these blocks with a single integer check, as shown in \Cref{alg:sel-fwd}.
Thus, selection attention does not concatenate selected KV blocks into a temporary tensor, and different heads do not need to be padded to the same selected length.
This two-level layout keeps compression attention compact, sliding-window attention contiguous, and selection attention fine-grained.

\begin{algorithm}[t]
\caption{Per-Head Selection Forward via Pointer Table}
\label{alg:sel-fwd}
\begin{algorithmic}[1]
\Require Queries $Q$, selected block ids $I[b, g, h, 1{:}K]$
\Require Selection mask $M[b, g, h, 1{:}K]$ \Comment{$M=1$ if the block is selected}
\Require Per-head GPU ptr table $\Pi_\ell$, blocks per chunk $B_{pc}$
\Require Query-group size $G_q$, softmax scale $\sigma$
\Statex
\State \textbf{// Step 1. Build per-$(b, g, h, t)$ pointer table.}
\For{each $(b, g, h, t)$ in parallel}
    \If{$M[b,g,h,t] = 0$}
        \State $P[b,g,h,t] \gets 0$ \Comment{not selected; skip in attention}
        \State \textbf{continue}
    \EndIf
    \State $c \gets I[b, g, h, t] \div B_{pc}$ \Comment{chunk id of the selected block}
    \State $o \gets I[b, g, h, t] \bmod B_{pc}$ \Comment{block offset within the chunk}
    \If{$\Pi_\ell[h, c] = 0$}
        \State \Call{WaitUntilReload}{$h,c$} \Comment{selected chunk is on CPU; wait until resident}
    \EndIf
    \State $P[b, g, h, t] \gets \Pi_\ell[h, c] + \text{offset}(o)$
\EndFor
\Statex
\State \textbf{// Step 2. Selection forward with online softmax.}
\For{each $(b, g, h)$ in parallel}
    \State $\mathbf{q} \gets Q[b,\ g\text{-th group},\ h]$ \Comment{$G_q$ query tokens}
    \State $m \gets -\infty; \quad d \gets 0; \quad \mathbf{a} \gets \mathbf{0}$ \Comment{online softmax state}
    \For{$t = 1, \dots, K$}
        \State $p \gets P[b, g, h, t]$
        \If{$p = 0$}
            \State \textbf{continue} \Comment{not selected or padded entry}
        \EndIf
        \State Load $\mathbf{K}_t, \mathbf{V}_t$ from address $p$
        \State $\mathbf{u} \gets \sigma \cdot \mathbf{q} \mathbf{K}_t^\top$
        \State $m' \gets \max(m,\ \max \mathbf{u})$
        \State $\mathbf{w} \gets \exp(\mathbf{u} - m'); \quad \alpha \gets \exp(m - m')$
        \State $d \gets d \cdot \alpha + \sum \mathbf{w}$
        \State $\mathbf{a} \gets \mathbf{a} \cdot \alpha + \mathbf{w} \mathbf{V}_t$
        \State $m \gets m'$
    \EndFor
    \State $O[b,\ g\text{-th group},\ h] \gets \mathbf{a} / d$
\EndFor
\end{algorithmic}
\end{algorithm}

\section{The Visualization of Attention Map}
\label{app:sec:attn_map}

\begin{figure*}[h]
  \centering
  \includegraphics[width=\linewidth]{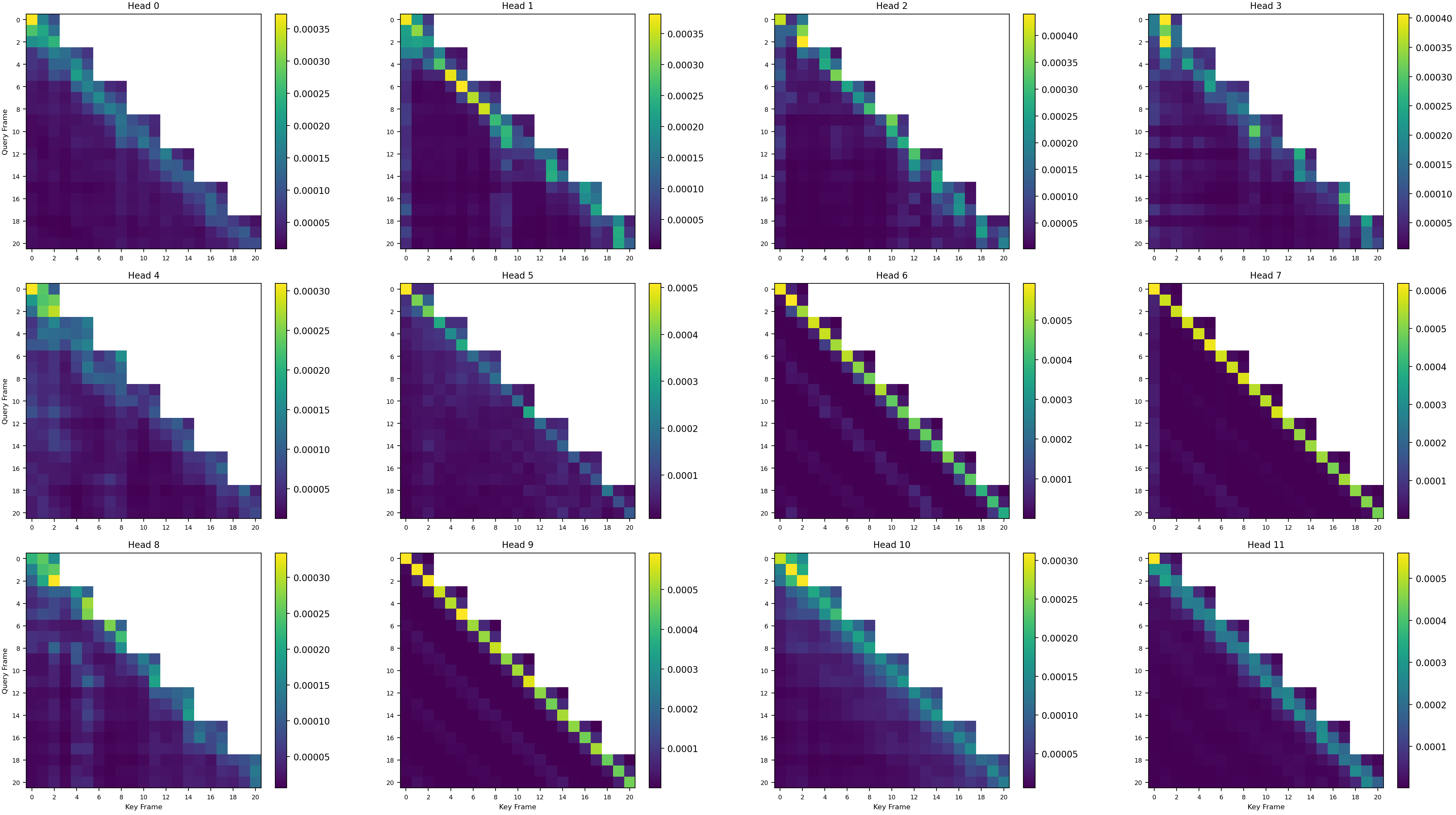}
\caption{
\textbf{Attention Map across different heads of Self-Forcing \cite{huang2025self} with full memory.}
}
\label{fig:attn_self_forcing}
\end{figure*}

We visualize the per-head attention patterns of an autoregressive video generation model to reveal how different heads attend to the historical context. We adopt Self-Forcing~\cite{huang2025self}, which retains the full historical KV cache without any compression. Therefore, the resulting attention maps faithfully reflect the intrinsic preference of each head.
We generate a $5$-second video and show the chunk-level attention map of each head in \cref{fig:attn_self_forcing}.

We observe two clear patterns: while some heads concentrate on recent frames, many others spread their attention across scattered positions in the history (\eg, Heads 0, 3, 4, 8, 10).
This confirms that (i) different heads exhibit distinct retrieval preferences (cross-head divergence), and (ii) a substantial fraction of heads rely on non-local historical context, which would be lost under sliding-window or sink-only memory.

\section{Short Video Generation.}
\label{app:sec:short_video}
\begin{table*}[h]
\centering
\caption{\textbf{Quantitative comparison on short video generation under the standard VBench~\cite{huang2024vbench} prompt suite.} We compare against other autoregressive models.}
\label{tab:short-vbench}
\begin{tabular}{l c c ccc}
\toprule
\multirow{2}{*}{Model} & \multirow{2}{*}{\#Params} & \multirow{2}{*}{Resolution} & \multicolumn{3}{c}{Vbench $\uparrow$} \\
\cline{4-6}
 &  &  & Total & Quality & Semantic \\
\midrule
CausVid~\cite{yin2025slow}                & 1.3B & $832{\times}480$ & 81.20          & 84.05          & 69.80 \\
Pyramid Flow~\cite{jin2024pyramidal}      & 2B   & $640{\times}384$ & 81.72          & 84.74          & 69.62 \\
MAGI-1~\cite{ai2025magi1autoregressivevideogeneration} & 4.5B & $832{\times}480$ & 79.18          & 82.04          & 67.74 \\
SkyReels-V2~\cite{chen2025skyreelsv2infinitelengthfilmgenerative}       & 1.3B & $960{\times}540$ & 82.67          & 84.70          & 74.53 \\
NOVA~\cite{deng2024nova}        & 0.6B & $768{\times}480$ & 80.12          & 80.39          & 79.05 \\
LongLive~\cite{yang2025longlive}          & 1.3B & $832{\times}480$ & 83.26 & 85.05 & 76.13 \\
MemFlow~\cite{ji2025memflow}              & 1.3B & $832{\times}480$ & 83.13          & 84.94          & 75.89 \\
Self-Forcing~\cite{huang2025self}         & 1.3B & $832{\times}480$ & \textbf{84.31}          & 85.07          & \textbf{81.28} \\
\midrule
\textbf{\methodAbbr}                      & 1.3B & $832{\times}480$ & 84.29             & \textbf{86.28}             & 76.31 \\
\bottomrule
\end{tabular}%
\end{table*}
For completeness, we also evaluate \methodAbbr~ on base 5-second video generation in \cref{tab:short-vbench} using the standard VBench~\cite{huang2024vbench} prompt suite.
Although our method is primarily designed for long-range memory retrieval, we verify that the proposed components also benefit for the short video quality.

Among these baselines, Self-Forcing~\cite{huang2025self} retains the full historical KV cache and serves as a strong reference for upper-bound quality on short videos.
\methodAbbr~achieves a Total VBench score of $84.29$, matching Self-Forcing ($84.31$) within $0.02$ points, while surpassing all other baselines that rely on memory compression or restricted context.
Notably, \methodAbbr~obtains the highest Quality score across the methods, suggesting that explicit per-head retrieval avoids the quality degradation.

\section{Extended Ablation Studies}
\label{app:sec:ablate}

\subsection{Impact of Pooling Size}

\begingroup
\setlength{\textfloatsep}{0pt}
\setlength{\intextsep}{0pt}

\begin{table*}[t]
\centering

\begin{minipage}[t]{0.48\textwidth}
\centering
\captionof{table}{
\textbf{Ablation of different KV block size $h \times w$}. $h$ and $w$ denote the height and width of each block.}
\label{tab:abl-pool}
\setlength{\tabcolsep}{6pt}
\renewcommand{\arraystretch}{1.15}
\small
\resizebox{\linewidth}{!}{%
\begin{tabular}{lcccc}
\toprule
\textbf{Pooling} & \textbf{IQ}  & \textbf{SC}  & \textbf{BC} & \textbf{CLIP} \\
\midrule
$15{\times}2$ (Top-$8$, 240 tokens)              & 71.42 & 98.41 & 97.71 & 27.12 \\
$15{\times}4$ (Top-$4$, 240 tokens)              & 71.02 & 98.31 & 97.04 & 26.28 \\
one frame (Top-$1$, 1560 tokens)                 & 69.98 & 98.30 & 97.01 & 26.71 \\
\bottomrule
\end{tabular}
}
\end{minipage}
\hfill
\begin{minipage}[t]{0.48\textwidth}
\centering
\captionof{table}{
\textbf{Ablation of Top-$K$ Selection.}
}
\label{tab:abl-topk}
\small
\resizebox{0.9\linewidth}{!}{%
\begin{tabular}{lcccc}
\toprule
\textbf{Top-$K$} & \textbf{IQ}  & \textbf{SC}  & \textbf{BC} & \textbf{CLIP} \\
\midrule
$K=1$              & 71.24 & 97.86 & 96.56 & 25.98 \\
$K=4$              & 71.38 & 98.34 & 97.30 & 26.53 \\
$K=8$              & 71.42 & 98.41 & 97.71 & 27.12 \\
$K=12$             & 72.31 & 98.96 & 97.69 & 27.44 \\
\bottomrule
\end{tabular}
}
\end{minipage}

\end{table*}

\endgroup
We ablate the pooling size in \cref{tab:abl-pool}.
Since our selection attention is implemented with Triton block-sparse kernels, the pooling size directly affects the fraction of valid tokens in each padded tile. As discussed in \cref{tab:bench_pooling_size} of the appendix, we restrict the search to configurations whose utilization exceeds $90\%$ for efficiency.
We compare three pooling configurations. The first two ($15{\times}2$ with Top-$8$ and $15{\times}4$ with Top-$4$) are matched to access the same total number of historical tokens ($240$), isolating the effect of pooling granularity from the token budget. The third uses one frame as the pooling unit (Top-$1$), which corresponds to the coarsest granularity and accesses substantially more tokens ($1560$).
Under the same token budget, $15{\times}2$ outperforms $15{\times}4$ on each metric, indicating that finer pooling provides more accurate selection than coarser pooling.
The frame-level configuration has the lowest IQ, SC, and BC despite accessing $6.5\times$ more tokens; its CLIP score lies between the two block-level configurations.

\subsection{Impact of Top-$K$ Selection}
Top-$K$ determines how many compressed KV blocks each head selects from the candidate pool per query group.
We sweep $k \in \{1, 4, 8, 12\}$ and finetune each variant for long video generation from the same short video base model. The results are reported in \cref{tab:abl-topk}.
Increasing $K$ generally improves quality. $K=12$ achieves the highest IQ, SC, and CLIP scores, while BC is slightly higher at $K=8$ ($97.71$ versus $97.69$).
The gains vary across metrics and intervals: CLIP increases by $0.55$, $0.59$, and $0.32$ points for $K=1\to4$, $4\to8$, and $8\to12$, respectively.
These results show the quality benefit of enlarging the retrieval budget over the tested range.

\subsection{The Efficiency–Memory Trade-off of LRU Size}
\label{sec:lru-ablation}
\begin{figure*}[t]
  \centering
  \subfloat[Inference Latency]{
        \includegraphics[width=0.32\linewidth]{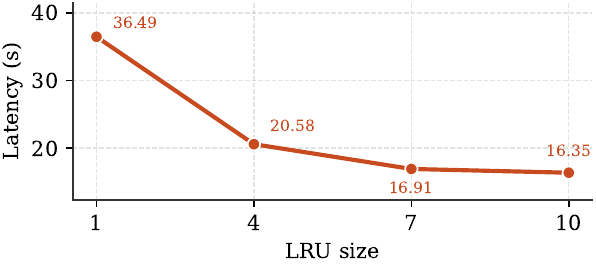}
    }
    \hfill
    \subfloat[Memory Consumption]{
        \includegraphics[width=0.32\linewidth]{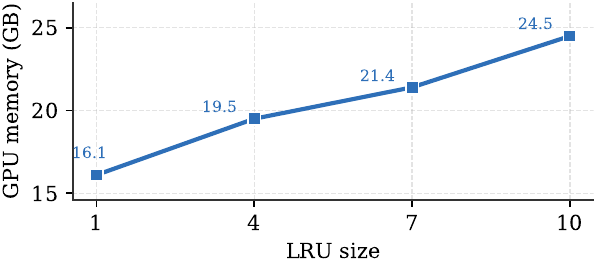}
    }
    \hfill
    \subfloat[Cache Hit Rate]{
        \includegraphics[width=0.32\linewidth]{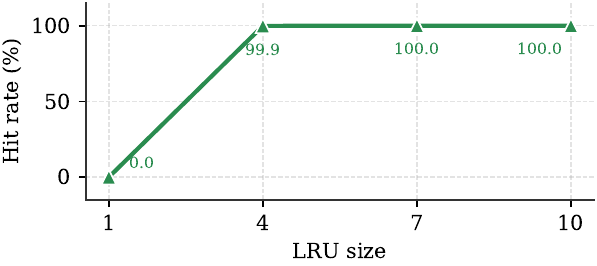}
    }
\caption{
\textbf{Impact of LRU cache size on 15-second video generation.} We compare (a) inference latency, (b) GPU memory consumption, and (c) cache hit rate under different LRU cache sizes.
}
\label{fig:ablate_lru_size}
\end{figure*}
The LRU cache stores recently retrieved KV blocks on the GPU, allowing them to be reused without reloading from CPU memory.
\cref{fig:ablate_lru_size} reports the impact of LRU size $G_{\text{LRU}}$ on inference latency, GPU memory, and cache hit rate.
With $G_{\text{LRU}}=1$, the cache is too small to retain frequently selected blocks, leading to a $0\%$ hit rate and a $2.2\times$ higher latency than the optimum due to constant CPU-GPU swapping.
Increasing $G_{\text{LRU}}$ to $4$ already reaches a $99.9\%$ hit rate, after which both latency and hit rate saturate while GPU memory keeps growing. From $G_{\text{LRU}}{=}7$ to $G_{\text{LRU}}{=}10$, latency improves by only $0.56$\,s ($16.91 \to 16.35$\,s), while memory inflates by over $3$\,GB ($21.4 \to 24.5$\,GB).
This rapid saturation is a direct consequence of our explicit retrieval design.
Since each query group selects only a small set of historical blocks, the working set of frequently accessed blocks is naturally compact, so a moderate LRU cache is sufficient to cover it.

\section{Additional Visualization for Length Extrapolation}
\label{sec:analysis}
\Cref{fig:wo_finetune} provides the qualitative example discussed in \cref{sec:discussion}: a 20-second video generated by a model trained only on 5-second clips, without long-video finetuning.
The fusion-gate visualization and the corresponding analysis are presented in the main text (\cref{fig:gate_ac_head}).

\begin{figure*}[t]
  \centering
  \includegraphics[width=\linewidth]{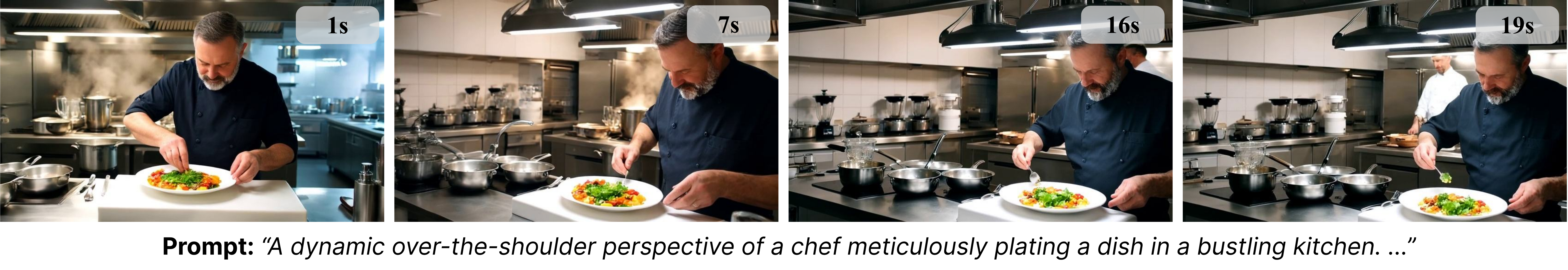}
\caption{
\textbf{Long-video generation by \methodAbbr{} without long-video finetuning.} A 20-second video generated by \methodAbbr{}, which is trained only on 5-second clips.
}
\label{fig:wo_finetune}
\end{figure*}

\section{Padding Efficiency for Different Pooling Sizes}
\begin{table}[h]
\centering
\caption{
\textbf{Best pooling block shapes for the Wan-2.1-1.3B\cite{wan2025wanopenadvancedlargescale} 480P configuration with latent shape $h=30,w=52$.}
The reported configurations use power-of-two padded tiles of at least 16 tokens.
Tile utilization (util.) is block size divided by padded size. The best tile-fit group achieves 93.8\% utilization with low padding overhead.
}
\label{tab:bench_pooling_size}
\begin{tabular}{lcccccc}
\toprule
tile fit & $h$ & $w$ & block size & padded & util. & blocks \\
\midrule
\multirow{2}{*}{best}
& 15 & 2  & 30  & 32  & 93.8\% & 52 \\
& 15 & 4  & 60  & 64  & 93.8\% & 26 \\
\midrule
\multirow{3}{*}{good}
& 2  & 13 & 26  & 32  & 81.3\% & 60 \\
& 2  & 26 & 52  & 64  & 81.3\% & 30 \\
& 2  & 52 & 104 & 128 & 81.3\% & 15 \\
\midrule
\multirow{3}{*}{ok}
& 3  & 4  & 12  & 16  & 75.0\% & 130 \\
& 6  & 2  & 12  & 16  & 75.0\% & 130 \\
& 6  & 4  & 24  & 32  & 75.0\% & 65 \\
\midrule
\multirow{6}{*}{poor}
& 5  & 2  & 10  & 16  & 62.5\% & 156 \\
& 5  & 4  & 20  & 32  & 62.5\% & 78 \\
& 10 & 2  & 20  & 32  & 62.5\% & 78 \\
& 3  & 13 & 39  & 64  & 60.9\% & 40 \\
& 10 & 4  & 40  & 64  & 62.5\% & 39 \\
& 6  & 13 & 78  & 128 & 60.9\% & 20 \\
\midrule
\multirow{3}{*}{terrible}
& 3  & 2  & 6   & 16  & 37.5\% & 260 \\
& 5  & 13 & 65  & 128 & 50.8\% & 24 \\
& 10 & 13 & 130 & 256 & 50.8\% & 12 \\
\bottomrule
\end{tabular}
\end{table}

Our Triton selection-attention kernel pads each block to a power-of-two tile. Tile utilization is the ratio of the number of valid tokens to the padded tile size. 
We ablate block size choices in \cref{tab:bench_pooling_size}. While smaller blocks (\eg, 12) provide finer-grained selection, they have a lower fraction of valid tokens per padded tile. The larger blocks (\eg, 60, 120) reduce sequence length more aggressively but may lose temporal/spatial granularity.

\section{Explicit Retrieval in Chunk-Based Video Generation}
\label{app:sec:gqa}

We highlight that our \methodAbbr~is specifically designed for chunk-based AR video generation, which significantly differs from sparse retrieval of full-resolution KV blocks \cite{yuan2025native} in LLMs. Specifically, in standard LLM decoding, each step involves a single query token, and modern architectures frequently employ Grouped Query Attention (GQA), wherein multiple query heads share a single KV head and can therefore share retrieval decisions with negligible overhead. By contrast, each AR step in video generation processes thousands of query tokens simultaneously, and prevalent video DiTs adopt Multi-Head Attention (MHA), where each query head maintains a dedicated KV head. These architectural differences fundamentally shift the computational bottleneck of sparse retrieval.

Therefore, in token-wise decoding, sparse retrieval primarily reduces the number of historical tokens attended to at each step. By contrast, in our chunk-based video generation,   sparse retrieval must additionally govern the union of selected blocks across all query tokens and attention heads, which is substantially more complex with unique challenges including local bias  and union explosion.
\methodAbbr~addresses these challenges to make explicit memory retrieval practical at both the algorithmic and system levels. Specifically, Adaptive Window Exclusion refines block selection by discounting recency bias in top-$K$ scoring; query sharing mitigates cross-query selection diversity by aggregating query tokens into representative proxies before retrieval; and Per-Head Scattered KV Access circumvents the need to expand head-specific selections into a large, head-aligned KV buffer, substantially reducing memory overhead.

\section{Permutation Equivariance of Reorder with RoPE}
\label{sec:reorder}
We justify the token reordering step in our compression pipeline.
Our compression operates on \emph{2D spatial neighborhoods}: tokens that are spatially adjacent in the original frame layout are pooled together so that each compressed block summarizes a semantically coherent region. However, when video tokens are flattened into a 1D sequence following the standard $f \times h \times w$ raster order, the tokens belonging to the same 2D spatial block are \emph{not} contiguous in memory, making block-wise pooling difficult to express as an efficient operation.

To address this, we introduce a one-time token reorder $\pi$ that places all tokens of each 2D spatial block into a contiguous range, so that compression becomes a simple contiguous reshape and mean reduction. 
The natural concern is whether this reorder with the corresponding RoPE  positions changes the attention output. 
We show below that reordering tokens and their RoPE positions \cite{su2024roformer} jointly leaves the attention output unchanged up to the same reordering, so the entire reorder–compress–attend–restore pipeline is mathematically equivalent to the unreordered computation.

\paragraph{Setup.}
Let $X \in \mathbb{R}^{N \times d}$ be a token sequence and $\Theta = (\theta_1, \ldots, \theta_N)$ the corresponding RoPE positional indices, where each $\theta_i$ encodes the spatiotemporal position $(f_i, h_i, w_i)$. RoPE-augmented attention is
\begin{equation}
    \mathrm{Attn}(X, \Theta) \;=\; \mathrm{softmax}\!\Big(\tfrac{Q'(K')^\top}{\sqrt{d}}\Big)\, V,
\end{equation}
where the per-token queries, keys, and values are
\begin{equation}
    Q'_i = R(\theta_i)\, W_Q X_i, \quad
    K'_j = R(\theta_j)\, W_K X_j, \quad
    V_j  = W_V X_j,
\end{equation}
and $R(\theta)$ is the RoPE rotation associated with position $\theta$. Let $\pi \in S_N$ be a permutation, and $P_\pi \in \{0,1\}^{N \times N}$ its permutation matrix, satisfying $(P_\pi X)_i = X_{\pi(i)}$ and $(P_\pi \Theta)_i = \theta_{\pi(i)}$.

\begin{proposition}[Joint reorder equivariance]\label{prop:rope-equiv}
For any permutation $\pi \in S_N$,
\begin{equation}
    \mathrm{Attn}(P_\pi X,\, P_\pi \Theta) \;=\; P_\pi \cdot \mathrm{Attn}(X, \Theta).
    \label{eq:equiv}
\end{equation}
\end{proposition}

\begin{proof}
The reordered query satisfies
\begin{equation}
    Q'^{(\pi)}_i \;=\; R\big((P_\pi\Theta)_i\big)\, W_Q\, (P_\pi X)_i \;=\; R(\theta_{\pi(i)})\, W_Q\, X_{\pi(i)} \;=\; Q'_{\pi(i)},
\end{equation}
so $Q'^{(\pi)} = P_\pi Q'$. The same holds for $K'$ and $V$, giving $K'^{(\pi)} = P_\pi K'$ and $V^{(\pi)} = P_\pi V$.
Using $(P_\pi K')^\top = (K')^\top P_\pi^\top$, the score matrix becomes
\begin{equation}
    Q'^{(\pi)} \big(K'^{(\pi)}\big)^\top
    \;=\; P_\pi\, Q' (K')^\top\, P_\pi^\top.
\end{equation}
Since row-wise softmax commutes with row permutation,
\begin{equation}
    \mathrm{softmax}\!\big(P_\pi M P_\pi^\top\big) \;=\; P_\pi\,\mathrm{softmax}(M)\, P_\pi^\top
\end{equation}
for any matrix $M$. Combining this with $V^{(\pi)} = P_\pi V$ and $P_\pi^\top P_\pi = I$,
\begin{align}
    \mathrm{Attn}(P_\pi X, P_\pi \Theta)
    &= \mathrm{softmax}\!\Big(\tfrac{P_\pi Q'(K')^\top P_\pi^\top}{\sqrt{d}}\Big)\, P_\pi V \\
    &= P_\pi\, \mathrm{softmax}\!\Big(\tfrac{Q'(K')^\top}{\sqrt{d}}\Big)\, P_\pi^\top P_\pi\, V \\
    &= P_\pi\, \mathrm{Attn}(X, \Theta). \qedhere
\end{align}
\end{proof}

\begin{corollary}[Reorder–compress equivalence]\label{cor:reorder-compress}
Let the reorder $\pi$ be chosen such that, for every 2D spatial compression block $b$, the indices $\{i : (h_i, w_i) \in b\}$ form a contiguous range under $\pi$. Then mean-pooling every $|b|$ consecutive tokens of $P_\pi K'$ and $P_\pi V$ yields the same compressed key/value sequence as applying 2D spatial block-pooling to the original $K'$ and $V$. Combined with Proposition~\ref{prop:rope-equiv}, the entire \emph{(reorder $\to$ contiguous compress $\to$ attention $\to$ inverse reorder)} pipeline is equivalent to \emph{(2D-block compress $\to$ attention)} on the unreordered sequence.
\end{corollary}

\begin{remark}
The corollary justifies our implementation: we apply \texttt{token\_reorder} once at the input and \texttt{token\_restore} once at the output. All transformer blocks (compression / selection / sliding-window attention) operate on reordered tokens internally, yielding strictly equivalent computation while reducing the spatial compression to a contiguous \texttt{reshape}+\texttt{mean} operation on GPU memory.
\end{remark}

\section{More Qualitative Results}
\label{sec:supp:more_qualitative}
\paragraph{Long Video Generation.}
\cref{fig:app_long1} shows 60-second single-prompt results. SWA suffers from gradual subject drift as the kitten's appearance changes across frames, since distant context is discarded. LongLive~\cite{yang2025longlive} preserves identity via attention sinks but produces repetitive frames (red boxes). \methodAbbr~maintains both subject consistency and natural scene evolution.
\paragraph{Multi-Prompt Long Video Generation.}
\cref{fig:app_inter2,fig:app_inter3} show multi-prompt 60-second results. 
LongLive~\cite{yang2025longlive} and MemFlow~\cite{ji2025memflow} both exhibit repetition (red boxes) for the same underlying reason: neither can attach to the full-resolution historical context. LongLive uses fixed sink tokens, while MemFlow maintains a fixed memory bank, both losing the fully fine-grained details needed to follow prompt transitions. \methodAbbr~retrieves full-resolution historical context on demand, faithfully following prompt transitions while preserving consistency.

\begin{figure*}[t]
  \centering
  \includegraphics[width=\linewidth]{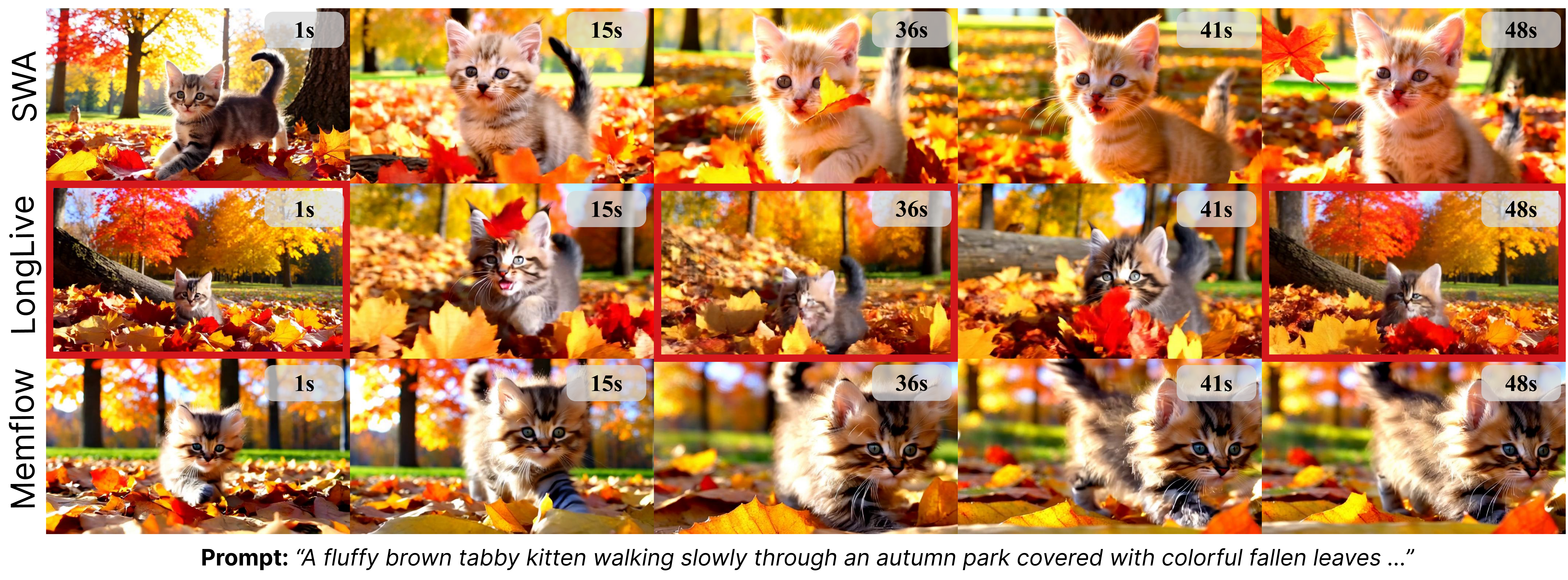}
\caption{
\textbf{Qualitative comparison on 60-second long video generation.} SWA denotes the sliding-window-only baseline. Red boxes highlight visible artifacts.}
\label{fig:app_long1}
\end{figure*}
\begin{figure*}[t]
  \centering
  \includegraphics[width=\linewidth]{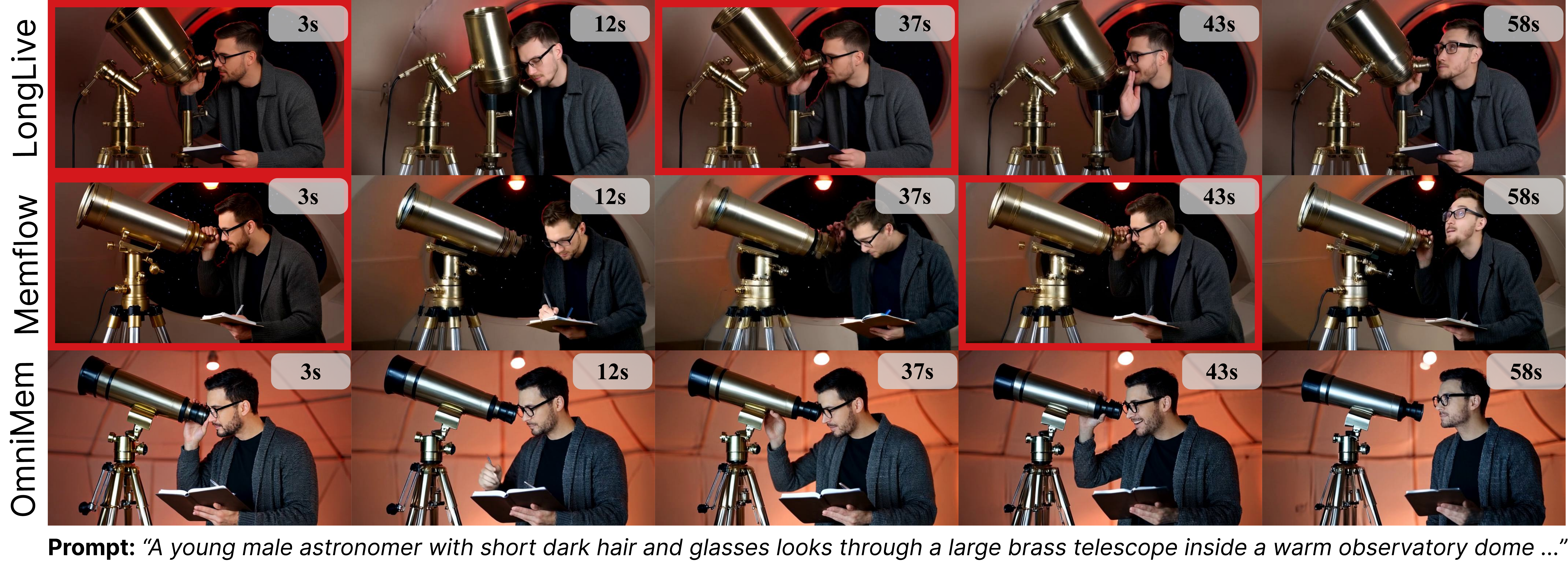}
\caption{
\textbf{Qualitative comparison on multi-prompt 60-second video generation.} Red boxes mark repetitive frames in LongLive \cite{yang2025longlive} and MemFlow \cite{ji2025memflow}.
}
\label{fig:app_inter3}
\end{figure*}
\begin{figure*}[t]
  \centering
  \includegraphics[width=\linewidth]{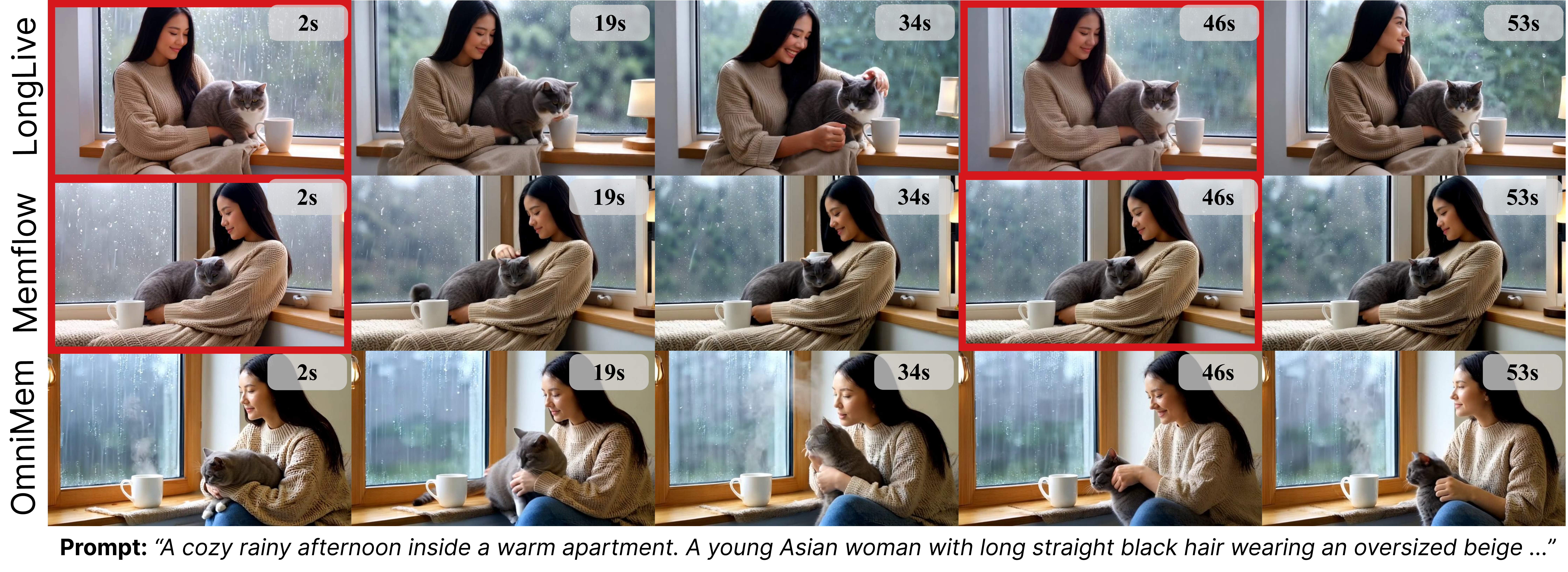}
\caption{
\textbf{Qualitative comparison on multi-prompt 60-second video generation.} Red boxes mark repetitive frames in LongLive \cite{yang2025longlive} and MemFlow \cite{ji2025memflow}.
}
\label{fig:app_inter2}
\end{figure*}

\clearpage

\clearpage

\section{Protocol for the NSA Comparison}
\label{app:sec:nsa_protocol}

We compare \methodAbbr{} with an adapted NSA~\cite{yuan2025native} baseline on 60-second video generation with prompt switching.
Both methods use the same backbone and training configurations for ODE initialization, Self-Forcing, and long-video tuning.
To accommodate the visual token layout, we replace NSA's one-dimensional pooling with two-dimensional spatial pooling.
The adapted baseline retains NSA's selection rule without Adaptive Window Exclusion or Query-Shared KV Selection.

\para{Training memory management}
The adapted NSA baseline exceeds the memory capacity of an NVIDIA H200 during training before reaching a rollout equivalent to 60 seconds of video.
Training must retain gradient-related activations in addition to the historical KV cache, while the union of selected blocks can grow across the queries in each video chunk.
To make this comparison feasible, we equip the baseline with the same per-head chunked KV management and asynchronous CPU offloading infrastructure used by \methodAbbr{}.
This infrastructure manages KV residency and data movement while retaining the baseline's attention and block-selection rules.
The quality comparison therefore evaluates the retrieval designs under matched training configurations.

\para{Inference memory measurement}
Peak inference VRAM is measured on a single NVIDIA H200.
For this measurement, the adapted NSA baseline runs \emph{without} CPU offloading, since its inference computation fits on the H200.
The measured peaks are 138.4\,GB for the adapted NSA baseline and 23.3\,GB for \methodAbbr{}.
Under this 60-second prompt-switching protocol, the baseline and \methodAbbr{} obtain Imaging Quality scores of 71.32 and 73.85, Subject Consistency scores of 95.23 and 97.72, and Background Consistency scores of 95.74 and 96.55, respectively.

\section{Human Evaluation Protocol}
\label{app:sec:human_eval}

We evaluate 60-second videos generated from 50 randomly sampled prompts under the prompt-switching setting.
For each prompt, we compare \methodAbbr{} with LongLive~\cite{yang2025longlive} and MemFlow~\cite{ji2025memflow}.
The three videos are presented anonymously in a randomly shuffled order.
Fifteen evaluators each assess all 50 triplets and select the best video separately for each of three criteria, yielding 750 judgments per criterion.

\para{Evaluation criteria}
For \emph{text alignment}, evaluators judge whether the video follows the supplied prompts.
For \emph{overall quality}, they assess the completeness of generated objects, meaningful and dynamic motion, and the presence of flickering or other artifacts.
For \emph{long-range consistency}, they assess whether subject identity, scene layout, and visual style remain consistent throughout the 60-second video, including when a subject reappears after an extended absence.

\para{Preference rates}
A method's preference rate is the fraction of the 750 judgments selecting that method for the corresponding criterion.
These are three-way preference rates, rather than pairwise win rates.
\methodAbbr{} receives 60.8\% of preferences for text alignment, 65.7\% for overall quality, and 56.1\% for long-range consistency.
The corresponding rates are 14.9\%, 12.4\%, and 25.9\% for LongLive, and 24.3\%, 21.9\%, and 18.0\% for MemFlow.

\section{Semantic Relevance of Retrieved History}
\label{app:sec:retrieval_relevance}

We examine whether the selected historical content is semantically relevant at both the frame and block levels.
The analysis uses 40 generated videos, each comprising 27 chunks, with prompt switches at chunks 9 and 18.
During every clean-context pass, we record the raw Top-12 block selections, using the default selection budget, from all 12 attention heads in layers 2, 15, and 29.
We evaluate chunks indexed 4--26, which have historical context outside the local window.

\para{Frame-level prompt alignment}
For each generated chunk, we aggregate block-selection counts by historical frame, excluding frames within the local window.
Heavily selected frames receive approximately six times as many block selections per frame as lightly selected frames.
We evaluate alignment with the current segment's prompt using the CLIP~\cite{CLIP} text--image score, computed from the decoded RGB frame and the prompt text.
As shown in \cref{tab:app:retrieval_clip}, heavily selected frames have higher prompt alignment at each inspected layer.

\begin{table}[htbp]
\centering
\caption{\textbf{Prompt alignment of retrieved historical frames.} CLIP text--image scores are reported as mean $\pm$ standard deviation. Higher values indicate stronger alignment with the current prompt.}
\label{tab:app:retrieval_clip}
\begin{tabular}{lcc}
\toprule
Layer & Heavily selected frames & Lightly selected frames \\
\midrule
2  & $0.2090 \pm 0.0055$ & $0.1953 \pm 0.0055$ \\
15 & $0.2111 \pm 0.0051$ & $0.1962 \pm 0.0060$ \\
29 & $0.2089 \pm 0.0055$ & $0.1942 \pm 0.0055$ \\
\bottomrule
\end{tabular}
\end{table}

\para{Block-level semantic similarity}
We next compare selected and unselected blocks within the same historical frame.
We apply DINOv2-B/14~\cite{dinov2} to decoded RGB frames to extract patch features.
For each current query block and historical candidate block, we average the patch features whose spatial receptive fields fall within its corresponding pixel region, then apply $\ell_2$ normalization.
We compute cosine similarity between the query-block representation and the selected historical blocks, and compare it with the similarity to the remaining, unselected blocks of the same historical frame.
\Cref{tab:app:retrieval_dino} shows higher similarity for selected blocks across all three layers.
Together, these analyses support the semantic relevance of the retrieved history.

\begin{table}[htbp]
\centering
\caption{\textbf{Semantic similarity of retrieved historical blocks.} DINOv2 cosine similarities are reported as mean $\pm$ standard deviation. Unselected blocks are drawn from the same historical frame as the selected blocks.}
\label{tab:app:retrieval_dino}
\begin{tabular}{lcc}
\toprule
Layer & Selected blocks & Unselected blocks \\
\midrule
2  & $0.731 \pm 0.011$ & $0.352 \pm 0.010$ \\
15 & $0.776 \pm 0.009$ & $0.367 \pm 0.009$ \\
29 & $0.757 \pm 0.010$ & $0.357 \pm 0.010$ \\
\bottomrule
\end{tabular}
\end{table}

\section{Generation Beyond 60 Seconds}
\label{app:sec:long120}

We compare \methodAbbr{} with LongLive~\cite{yang2025longlive} on 120-second video generation with prompt switching.
As shown in \cref{tab:app:long120}, \methodAbbr{} achieves higher Imaging Quality, Subject Consistency, and Background Consistency in this longer setting.
The improvements are 1.87, 1.28, and 1.57 points, respectively.

\begin{table}[htbp]
\centering
\caption{\textbf{120-second video generation with prompt switching.} Higher values are better for all three metrics.}
\label{tab:app:long120}
\begin{tabular}{lccc}
\toprule
Method & Imaging Quality & Subject Consistency & Background Consistency \\
\midrule
LongLive & 70.86 & 96.10 & 94.37 \\
\methodAbbr{} & \textbf{72.73} & \textbf{97.38} & \textbf{95.94} \\
\bottomrule
\end{tabular}
\end{table}

\clearpage
\section*{NeurIPS Paper Checklist}

\begin{enumerate}

\item {\bf Claims}
    \item[] Question: Do the main claims made in the abstract and introduction accurately reflect the paper's contributions and scope?
    \item[] Answer: \answerYes{} 
    \item[] Justification: The abstract and introduction clearly state the contributions and scope of the paper, which match the theoretical and experimental results presented.


\item {\bf Limitations}
    \item[] Question: Does the paper discuss the limitations of the work performed by the authors?
    \item[] Answer: \answerYes{} 
    \item[] Justification: Limitations of our approach are discussed in the conclusion/discussion section.


\item {\bf Theory assumptions and proofs}
    \item[] Question: For each theoretical result, does the paper provide the full set of assumptions and a complete (and correct) proof?
    \item[] Answer: \answerYes{} 
    \item[] Justification: No theoretical results state their assumptions, and no proofs are needed. 

    \item {\bf Experimental result reproducibility}
    \item[] Question: Does the paper fully disclose all the information needed to reproduce the main experimental results of the paper to the extent that it affects the main claims and/or conclusions of the paper (regardless of whether the code and data are provided or not)?
    \item[] Answer: \answerYes{} 
    \item[] Justification: We describe the model architecture, training procedure, hyperparameters, and datasets used in the main paper and appendix in sufficient detail to reproduce the main results.


\item {\bf Open access to data and code}
    \item[] Question: Does the paper provide open access to the data and code, with sufficient instructions to faithfully reproduce the main experimental results, as described in supplemental material?
    \item[] Answer: \answerYes{} 
    \item[] Justification: Code will be provided and published along with instructions for reproducing the main results. All datasets used are publicly available.


\item {\bf Experimental setting/details}
    \item[] Question: Does the paper specify all the training and test details (e.g., data splits, hyperparameters, how they were chosen, type of optimizer) necessary to understand the results?
    \item[] Answer: \answerYes{} 
    \item[] Justification: Training and evaluation details, including data splits, optimizer, hyperparameters, and selection procedure, are specified in the experiments section and appendix.


\item {\bf Experiment statistical significance}
    \item[] Question: Does the paper report error bars suitably and correctly defined or other appropriate information about the statistical significance of the experiments?
    \item[] Answer: \answerNA{} 
    \item[] Justification: We report results averaged over multiple random seeds.


\item {\bf Experiments compute resources}
    \item[] Question: For each experiment, does the paper provide sufficient information on the computer resources (type of compute workers, memory, time of execution) needed to reproduce the experiments?
    \item[] Answer: \answerYes{} 
    \item[] Justification: Compute resources, including GPU type and approximate runtime, are reported.

    
\item {\bf Code of ethics}
    \item[] Question: Does the research conducted in the paper conform, in every respect, with the NeurIPS Code of Ethics \url{https://neurips.cc/public/EthicsGuidelines}?
    \item[] Answer: \answerYes{} 
    \item[] Justification: We have reviewed the NeurIPS Code of Ethics and confirm that the research conforms with it in every respect.


\item {\bf Broader impacts}
    \item[] Question: Does the paper discuss both potential positive societal impacts and negative societal impacts of the work performed?
    \item[] Answer: \answerYes{} 
    \item[] Justification: Potential positive and negative societal impacts of the work are discussed in the broader impacts section.

    
\item {\bf Safeguards}
    \item[] Question: Does the paper describe safeguards that have been put in place for responsible release of data or models that have a high risk for misuse (e.g., pre-trained language models, image generators, or scraped datasets)?
    \item[] Answer: \answerNA{} 
    \item[] Justification: The paper does not release data or models with a high risk for misuse.


\item {\bf Licenses for existing assets}
    \item[] Question: Are the creators or original owners of assets (e.g., code, data, models), used in the paper, properly credited and are the license and terms of use explicitly mentioned and properly respected?
    \item[] Answer: \answerYes{} 
    \item[] Justification: All datasets, code, and models used are properly cited, and their licenses and terms of use are respected.


\item {\bf New assets}
    \item[] Question: Are new assets introduced in the paper well documented and is the documentation provided alongside the assets?
    \item[] Answer: \answerNA{} 
    \item[] Justification: No new assets introduced in the paper.


\item {\bf Crowdsourcing and research with human subjects}
    \item[] Question: For crowdsourcing experiments and research with human subjects, does the paper include the full text of instructions given to participants and screenshots, if applicable, as well as details about compensation (if any)? 
    \item[] Answer: \answerNA{} 
    \item[] Justification: The paper does not involve crowdsourcing or research with human subjects.


\item {\bf Institutional review board (IRB) approvals or equivalent for research with human subjects}
    \item[] Question: Does the paper describe potential risks incurred by study participants, whether such risks were disclosed to the subjects, and whether Institutional Review Board (IRB) approvals (or an equivalent approval/review based on the requirements of your country or institution) were obtained?
    \item[] Answer: \answerNA{} 
    \item[] Justification: The paper does not involve research with human subjects.


\item {\bf Declaration of LLM usage}
    \item[] Question: Does the paper describe the usage of LLMs if it is an important, original, or non-standard component of the core methods in this research? Note that if the LLM is used only for writing, editing, or formatting purposes and does \emph{not} impact the core methodology, scientific rigor, or originality of the research, declaration is not required.
    \item[] Answer: \answerNA{} 
    \item[] Justification: LLMs are not used as an important, original, or non-standard component of the core methods in this research.


\end{enumerate}

\end{document}